\documentclass[transaction]{IEEEtran}
\usepackage{amsfonts}
\usepackage{amssymb}
\usepackage{hyperref}
\usepackage{graphicx}
\usepackage{enumerate}
\usepackage{amsmath}
\usepackage{color}
\usepackage{amsthm}
\usepackage{verbatim}
\usepackage{amsmath}
\usepackage{algorithm}
\usepackage{cite}
\usepackage{algpseudocode}
\usepackage{subfig}
\usepackage{graphicx}
\usepackage{float}
\usepackage{autobreak}
\usepackage{booktabs}  
\usepackage{array}   
\usepackage{xcolor}  
\usepackage{amsfonts,amssymb}
\providecommand{\keywords}[1]{\textbf{\textit{Index terms---}} #1}

\newcommand{\bp}{\begin{proof} \small }
\newcommand{\ep}{\end{proof} \normalsize}
\newcommand{\epx}{\end{proof} \small}
\newcommand{\bpa}{\begin{proofappx} \footnotesize }
\newcommand{\epa}{\end{proofappx} \small }
\newtheorem{theorem}{Theorem}

\newtheorem{assumption}{Assumption}

\newtheorem*{theorem*}{Theorem}
\newtheorem*{proposition*}{Proposition}
\newtheorem*{corollary*}{Corollary}
\newtheorem*{lemma*}{Lemma}
\newtheorem*{assumption*}{Assumption}
\newtheorem*{definition*}{Definition}
\newtheorem*{claim*}{Claim}

\newcommand{\be}{\begin{equation}}
\newcommand{\ee}{\end{equation}}
\newcommand{\bs}{\begin{subequations}}
\newcommand{\es}{\end{subequations}}
\newcommand{\bq}{\begin{eqnarray}}
\newcommand{\eq}{\end{eqnarray}}
\newcommand{\bqn}{\begin{eqnarray*}}
\newcommand{\eqn}{\end{eqnarray*}}

\newcommand{\ba}{\left[ \begin{array}}
\newcommand{\ea}{\\ \end{array} \right]}
\newcommand{\ben}{\begin{enumerate}}
\newcommand{\een}{\end{enumerate}}

\def\real{{\mathchoice%
{\hbox{\rm\setbox1=\hbox{I}\copy1\kern-.45\wd1 R}}
{\hbox{\rm\setbox1=\hbox{I}\copy1\kern-.45\wd1 R}}
{\hbox{\scriptsize\rm\setbox1=\hbox{I}\copy1\kern-.45\wd1 R}}
{\hbox{\scriptsize\rm\setbox1=\hbox{I}\copy1\kern-.45\wd1 R}}}}

\def\Zint{{\mathchoice{\setbox1=\hbox{\sf Z}\copy1\kern-.75\wd1\box1}
{\setbox1=\hbox{\sf Z}\copy1\kern-.75\wd1\box1}
{\setbox1=\hbox{\scriptsize\sf Z}\copy1\kern-.75\wd1\box1}
{\setbox1=\hbox{\scriptsize\sf Z}\copy1\kern-.75\wd1\box1}}}
\newcommand{\complex}{ \hbox{\rm C\kern-0.45em\rule[.07em]{.02em}{.58em}%
\kern 0.43em}}

\allowdisplaybreaks

\begin{document}
%

%
%
%
\title{Mitigating Modality Quantity and Quality Imbalance in Multimodal Online Federated Learning}

\author{{Heqiang Wang, Weihong Yang, Xiaoxiong Zhong, Jia Zhou, Fangming Liu, Weizhe Zhang} 
\thanks{
H. Wang, W. Yang, X. Zhong, J. Zhou,  F. Liu and W. Zhang are with Peng Cheng Laboratory, Shenzhen, 518066, China. (Corresponding Authors: Weihong Yang, Xiaoxiong Zhong.)} 
\thanks{Acknowledge grants: China Postdoctoral Science Foundation with Certificate Number: 2025M773494. Peng Cheng Laboratory Project (Grant No. PCL2025AS213), and the Shenzhen Science and Technology Program under grant JCYJ20220530143811027. }
} 

\maketitle

\begin{abstract}
The Internet of Things (IoT) ecosystem produces massive volumes of multimodal data from diverse sources, including sensors, cameras, and microphones. With advances in edge intelligence, IoT devices have evolved from simple data acquisition units into computationally capable nodes, enabling localized processing of heterogeneous multimodal data. This evolution necessitates distributed learning paradigms that can efficiently handle such data. Furthermore, the continuous nature of data generation and the limited storage capacity of edge devices demand an online learning framework. Multimodal Online Federated Learning (MMO-FL) has emerged as a promising approach to meet these requirements. However, MMO-FL faces new challenges due to the inherent instability of IoT devices, which often results in modality quantity and quality imbalance (QQI) during data collection. In this work, we systematically investigate the impact of QQI within the MMO-FL framework and present a comprehensive theoretical analysis quantifying how both types of imbalance degrade learning performance. To address these challenges, we propose the Modality Quantity and Quality Rebalanced (QQR) algorithm, a prototype learning based method designed to operate in parallel with the training process. Extensive experiments on two real-world multimodal datasets show that the proposed QQR algorithm consistently outperforms benchmarks under modality imbalance conditions with promising learning performance.
\end{abstract}

\keywords{Federated Learning, Multimodal Learning, Online Learning, Internet of Thing, Modality Imbalanced.}

%
\IEEEpeerreviewmaketitle

\section{Introduction }
The rapid growth of the Internet of Things (IoT) \cite{atzori2010internet} has resulted in an extraordinary increase in data generated by diverse interconnected devices, such as smart home systems \cite{wang2023local}, wearable health trackers \cite{wu2020rigid}, and industrial sensors \cite{wang2025denoising}. To enable intelligent applications and services within this ecosystem, artificial intelligence, particularly machine learning and deep learning, has become an essential approach for building models from large-scale IoT data. Traditionally, model training has been conducted on centralized cloud platforms or in data centers. However, as both the volume of IoT data and the number of connected devices continue to rise, this centralized paradigm encounters scalability and efficiency bottlenecks. Transmitting vast amounts of raw data to a central server requires substantial network bandwidth and induces high communication costs, making it unsuitable for latency-critical scenarios such as autonomous driving \cite{ma2023autors} and real-time health monitoring \cite{al2018context}. Moreover, sending sensitive information to the cloud raises significant privacy risks \cite{guo2023privacy}. With IoT devices evolving from basic data acquisition units into capable edge nodes, there is growing potential to utilize their local computing resources to address these challenges. In this regard, federated learning (FL) \cite{lim2020federated, mills2021multi, duan2020self} has emerged as an attractive distributed learning framework. FL facilitates collaborative model training across devices while keeping raw data local, offering a communication-efficient and privacy-preserving alternative to centralized training. By minimizing data transfer and maintaining data confidentiality, FL provides a scalable pathway for deploying intelligent IoT applications.

\begin{figure}[htp]
\vspace{-5pt}
\centering
\subfloat{\includegraphics[width=1\linewidth]{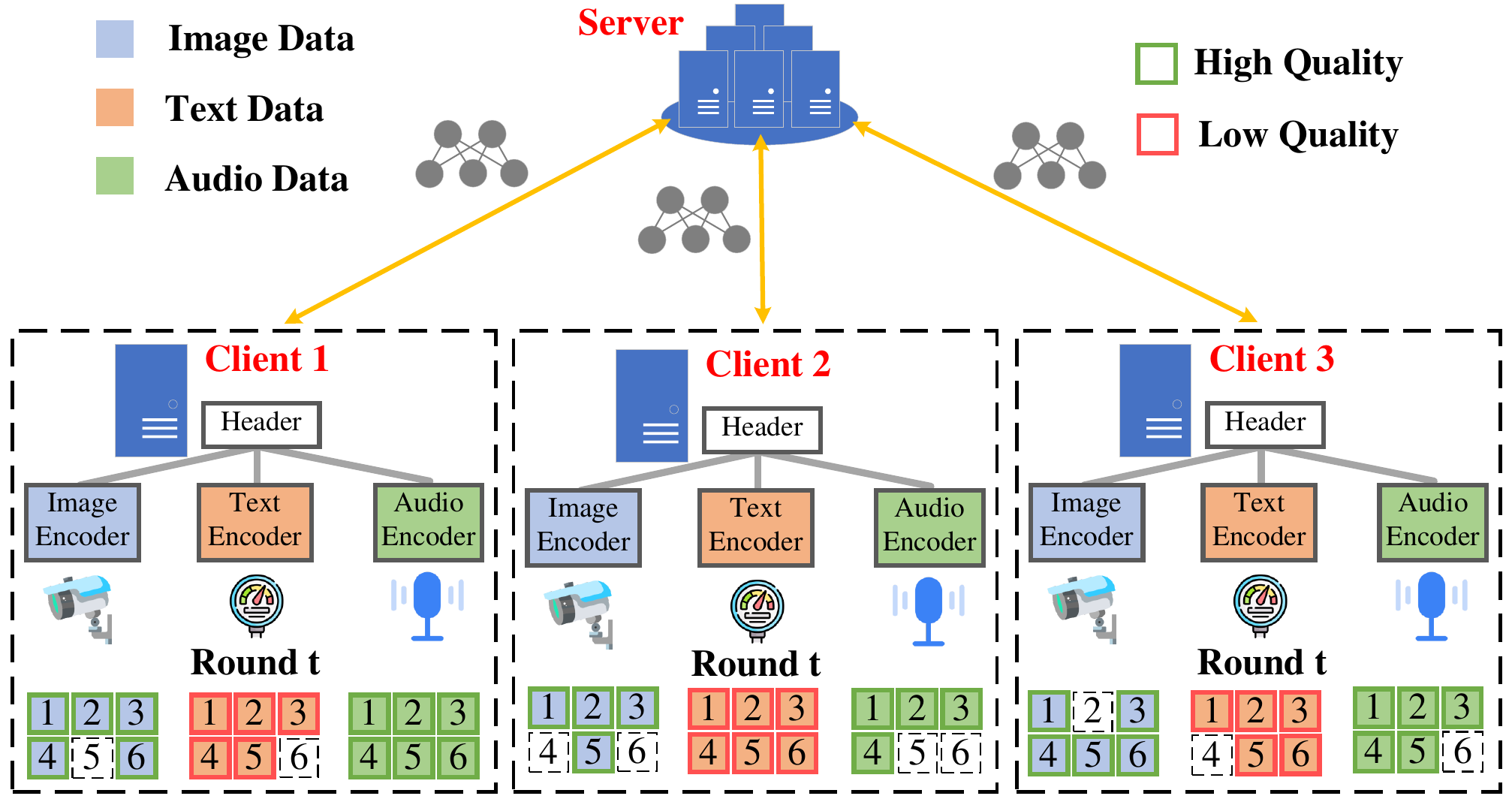}} 
\caption{IoT-Based MMO-FL with Modality Imbalance}  
\label{mmo_fl_ib}
\vspace{-5pt}
\end{figure}

Conventional FL frameworks for IoT have largely focused on unimodal data. In reality, however, IoT systems operate in inherently multimodal environments, where a variety of sensors capture data from different modalities \cite{huang2019multimodal}, such as visual information from cameras, audio signals from microphones, and structured or time-series data from other sensing devices. Such multimodal inputs provide richer and more comprehensive information for downstream tasks. To leverage this diversity, multimodal federated learning (MFL) \cite{che2023multimodal} has been introduced as an extension of FL, enabling collaborative model training across distributed multimodal data sources. In a typical MFL architecture, separate modality-specific encoders are employed, each designed to process a particular data type and extract discriminative features from high-dimensional raw inputs. The extracted features are then fused and passed to a shared head network, typically comprising deep neural layers and a softmax classifier, to produce the final prediction.

In conventional MFL, all training data are collected in advance, and model training proceeds in an offline manner. In contrast, real-world applications often require an online learning setting, where data are continuously acquired throughout the training process. Under such real-time and non-ideal conditions, variations in sensor performance and differences in modality characteristics can lead to modality imbalance during learning. In this work, modality imbalance is discussed in two primary forms: \textbf{Quantity Imbalance}: differences in sensor reliability or operation result in unequal amounts of data collected for each modality on a given client \cite{le2025multimodal}, and \textbf{Quality Imbalance}: variations in sensor capability cause differences in data integrity, with some modalities producing low-noise data while others suffer from higher noise levels \cite{zhang2024multimodal}. These issues are typically absent in offline settings, where all data are pre-collected, but they become especially pronounced in online scenarios. If not addressed, modality imbalance can severely degrade learning performance and, in extreme cases, prevent model convergence.

Motivated by the above challenges, this work focuses on addressing modality quantity and quality imbalance in the Multimodal Online Federated Learning (MMO-FL) framework \cite{wang2025multimodal}. To the best of our knowledge, this is the first systematic investigation of how multimodal imbalances affect learning performance in MMO-FL, examined from both theoretical and experimental perspectives. An overview of the MMO-FL framework with modality imbalance is shown in Fig.~\ref{mmo_fl_ib}. The quantity and quality of different modal data vary across clients and modalities at any given time, and these variations are inherently uncontrolled. The main contributions of this work are summarized as follows:
\begin{enumerate}
    \item  We study the problem of modality quantity and quality imbalance in MMO-FL scenarios, which primarily stems from the suboptimal conditions of data collection in online learning settings. Different sensors on clients acquire data from different modalities, and variations in their capabilities and operating frequencies result in disparities in both the quantity and quality of collected data. These imbalances ultimately impact the learning process.
    \item We present a detailed theoretical analysis of MMO-FL in the presence of modality quantity and quality imbalance. The derived regret bound explicitly captures the performance degradation caused by each imbalance type and offers insights for subsequent algorithm design.
    \item To address modality quantity and quality imbalance in the MMO-FL framework, we propose the Modality Quantity and Quality Rebalancing (QQR) algorithm. Built upon prototype learning, QQR provides efficient online rebalancing strategies that separately mitigate the effects of quantity and quality imbalance, thereby improving the learning performance of MMO-FL.
    \item We evaluate the proposed QQR algorithms within the MMO-FL framework on two multimodal datasets, UCI-HAR and MVSA-Single. The experimental results indicate that QQR outperforms benchmark methods in addressing both modality quantity and quality imbalance on both datasets. In addition, the results demonstrate strong generalization capability across diverse modality combinations in modality imbalance scenarios.
\end{enumerate}

The remainder of this paper is organized as follows. Section II reviews related work on multimodal federated learning, online federated learning, and modality quantity and quality imbalance. Section III describes the system model and problem formulation. Section IV outlines the MMO-FL workflow under modality quantity and quality imbalance. Section V presents the regret analysis of MMO-FL and evaluates the impact of these imbalances on learning performance. Section VI presents the proposed QQR algorithm. Section VII reports experimental results that validate the effectiveness of the proposed QQR algorithm within MMO-FL framework. Finally, Section VIII concludes the paper.

\section{Related Work }
\subsection{Multimodal Federated Learning}
MFL focuses on training task-specific models across clients that hold data from multiple modalities, enabling effective utilization of diverse and distributed information sources. As interest in MFL continues to grow, various algorithms have been proposed to address its unique challenges and improve model performance. One central issue is modality heterogeneity, where clients possess differing subsets of modalities, complicating model aggregation and limiting the efficiency of knowledge sharing. Several studies have investigated solutions for heterogeneous modality fusion to overcome this issue, including \cite{ouyang2023harmony, chen2024feddat, chen2022fedmsplit}. Another key challenge lies in selecting appropriate modalities for training under limited computational and communication resources. To address this, MPriorityFed \cite{bian2024prioritizing} introduces an adaptive resource allocation strategy that prioritizes modality encoders based on their utility and cost, improving efficiency. In addition to the above challenges, the problem of missing modalities has attracted increasing attention \cite{ma2021smil}. Missing data may result from incomplete data acquisition \cite{zhang2022m3care}, sensor malfunctions \cite{maheshwari2024missing}, or privacy-related restrictions \cite{ma2021smil}, all of which negatively affect model performance. For instance, the MFCPL framework \cite{le2025cross} employs cross-modal prototypes to facilitate knowledge transfer at both shared and modality-specific levels. However, most of the existing methods assume an offline learning scenario with access to a static dataset, limiting their applicability in real-world environments where data arrives sequentially and models must adapt continuously. Moreover, the missing modality problem \cite{wang2025multimodal} can be viewed as a specific case of the broader modality quantity imbalance addressed in this work. Unlike previous studies, we consider a more general setting where parts of a modality may be absent rather than the entire modality being unavailable. Furthermore, our approach incorporates the influence of modality quality on learning performance, a factor often overlooked in previous MFL works.

\subsection{Online Federated Learning }
Online learning processes data in a sequential manner, updating models incrementally, which makes it well-suited for scenarios involving continuous data streams and the need for real-time adaptation \cite{rang2021data}. Such methods are computationally efficient and remove the requirement of having the entire dataset beforehand, making them particularly advantageous in memory-constrained IoT settings. Within the FL paradigm, OFL extends these principles to distributed networks of decentralized learners \cite{hong2021communication}. Unlike traditional offline FL, OFL focuses on minimizing long-term cumulative regret rather than static optimization objectives during local updates. Although research on OFL is still relatively limited, several works have made notable progress. For example, \cite{kwon2023tighter} proposes a communication-efficient OFL approach that achieves strong performance while reducing communication costs, and \cite{mitra2021online} introduces FedOMD, an OFL method designed for uncertain environments that can handle streaming data without assuming prior knowledge of loss distributions. While most of these studies address the horizontal FL (HFL) setting, \cite{wang2023online} investigates the vertical FL (VFL) context, presenting an online VFL framework for cooperative spectrum sensing with sublinear regret guarantees. In the industrial domain, \cite{wang2025denoising} tackles noise interference and device heterogeneity challenges in online VFL systems. However, all these efforts focus on unimodal OFL. In practical IoT applications, multimodal data is common, requiring joint learning from multiple sensor types. Recent work \cite{wang2025multimodal} introduces the concept of MMO-FL and addresses the problem of missing modalities. It is worth noting that missing modalities can be seen as a special case of the quantity imbalance problem considered in this study, while our work investigates a more general setting that encompasses both quantity and quality imbalance.

\subsection{Modality Quality and Quantity Imbalance }
Multimodal learning often faces modality quality imbalance, where certain modalities provide high-quality signals, while others are degraded by noise or carry less informative content. To address this issue, work \cite{zhang2023provable} proposed a quality-aware multimodal fusion framework grounded in uncertainty estimation, achieving provably robust fusion performance under varying data quality conditions. Similarly, a comprehensive survey by \cite{zhang2024multimodal} highlights quality-varying multimodal data as a central challenge and categorizes existing approaches that tackle noise, missing modalities, and uneven signal quality. Recent studies have continued to explore inter-modal quality imbalance. For instance, \cite{fan2023pmr} and \cite{wei2024diagnosing} propose targeted methods to mitigate quality disparities across modalities, thereby improving overall multimodal learning performance. In parallel, growing attention has been given to modality quantity imbalance, where certain modalities are sparsely observed or underrepresented. \cite{fan2024overcome} examine this issue within the MFL context, identifying modal bias, a phenomenon in which clients possessing only unimodal data exert disproportionate influence on the global model. To address this, they introduce a balanced modality selection strategy to regulate client and modality participation, along with leveraging global prototypes for performance enhancement. Building on this work, their follow-up method, FedCMI \cite{fan2023balanced}, introduces cross-modal infiltration, which transfers knowledge from strong to weak modalities via projector modules and class-wise temperature scaling. While previous studies have addressed modality quality and quantity imbalance individually, as noted in the aforementioned works, these approaches are predominantly based on offline learning and do not account for the dynamic characteristics of online learning scenarios. Moreover, prior research has not examined the joint impact of modality quality and quantity imbalance, which introduces additional complexity and imposes stricter requirements on algorithm design. 

\section{System Model }
Before presenting the details of the system model, we provide a summary of the key notations in Table~\ref{table1}.

\begin{table}[ht]
\vspace{-5pt}
\centering
\caption{Key Notations}
\begin{tabular}{cc}
\toprule
\textbf{Symbol} & \textbf{Semantics} \\
\midrule
$K$ & The number of clients \\
$M$ & The number of modalities \\
$T$ & The number of global rounds \\
$N$ & The number of data samples collected by client\\
$E$ & The number of local iterations \\
$C$ & The number of classes \\
$\theta^m$ & The modality encoder\\
$\theta^0$ & The head encoder\\
$\eta$ & The learning rate \\
$Z$ & The feature extractor\\
$\Theta$ & The overall model \\
$P_k^{t, m}$ & The data availability status \\
$Q_k^{t, m}$ & The modality quality status \\
$\mathcal{D}_k^t$ & The local training dataset \\
$\textbf{G}^{t, \tau}_k$ & The gradient of the local overall model \\
$v_{k, c}^{t, m}$ & The local prototype \\
$v_{c}^{t, m}$ & The instantaneous global prototype \\
$\bar{v}_{c}^{t, m}$ & The cumulative global prototype \\
$\bar{\mathcal{V}}^{t}$ & The cumulative global prototype collection \\
$\mathcal{L}_{PCE}^m$ & The prototype cross entropy \\
$\mathcal{L}_{PLR}$ & The prototypical quality rebalancing loss \\
$\lambda_p$ & The intra-round quantity imbalance ratio \\
$\lambda_r$ & The inter-round quantity imbalance ratio \\
$\delta_r$ &  The intra-round quality imbalance ratio\\
\bottomrule
\end{tabular}
\label{table1}
\vspace{-5pt}
\end{table}

Consider an IoT-enabled smart factory comprising a cloud server and $K$ workstations serving as clients. Each workstation is equipped with multiple sensors to monitor factory conditions across various modalities at different locations, including visual, acoustic, and environmental sensors (e.g., temperature and humidity). With real-time data streaming from these sensors, the goal is to collaboratively train a unified global model using multimodal data distributed across clients. This configuration forms the basic setting for the MMO-FL problem. During factory operation, sensors at each client continuously gather new data over time, and the timeline is divided into discrete intervals denoted as $t = 1, 2, ..., T$. For simplicity, each interval is treated as a global training round.

In each global round, every client $k \in \mathcal{K}$ independently acquires data for each modality using its corresponding sensors. Let $N$ denote the target number of samples each client aims to collect per round. In practice, however, variations in sensor reliability and environmental factors can lead to modality quantity and quality imbalance, where some modalities fail to produce complete or valid samples. As a result, the actual local training dataset of client $k$ at round $t$ is defined as: 
\begin{align}
 & \mathcal{D}_k^t = \left (  \left [ \left ( X_k^{t, 1}, P_k^{t, 1}\right ) ,  \dots, \left ( X_k^{t, M}, P_k^{t, M} \right ) ; Q_k^{t}; Y_k^{t} \right ]  \right  ) \notag \\
& =  \left\{  \left\{ \left [ \left ( x_{k, n}^{t, m}, p_{k, n}^{t, m} \right ); y_{k, n}^{t} \right ] \right \}_{n = 1}^{N} ; q_k^{t, m} \right \}_{m = 1}^{M}  
\end{align}
Here, $x_{k, n}^{t, m}$ represents the $m^{th}$ modality data of the $n^{th}$ sample in client $k$ collected at global round $t$, and $y_{k, n}^{t}$ denotes the corresponding label. The binary indicator $p_{k, n}^{t, m} \in \left\{ 0, 1 \right\} $ specifies whether modality $m$ of the $n^{th}$ sample is available: a value of 1 indicates successful data collection, whereas 0 indicates that the modality is missing. Similarly, the binary indicator $q_{k}^{t, m} \in  \left\{ 0, 1 \right\} $ reflects the quality status of modality $m$, which can be determined by monitoring the operating condition of the client’s sensor. A value of 1 indicates that the collected modality data is of normal quality, while 0 signifies low quality. The overall dataset aggregated across all clients at global round $t$ is defined as $\mathcal{D}^t = \sum_{k=1}^K \mathcal{D}_k^t$. Due to the variability in $P_k^{t, m}$ and $Q_k^{t}$ across clients and rounds, the local datasets are inherently heterogeneous, leading to modality quality and quantity imbalance during training.

In the MFL, the global model trained collaboratively by the clients and aggregated on the server consists of two main components: the modality encoders $\theta^1, \dots, \theta^M$, which extract feature representations from raw inputs, and the head encoder $\theta^0$, which integrates these features to generate the final output. The architecture of each modality encoder is tailored to its specific data type. For client $k$ at round $t$, the feature vector for the $m^{th}$ modality, extracted by modality encoder $\theta^m$, is denoted as $ Z_k^{t, m} = \theta^m(X_k^{t, m} )$. The header encoder $\theta^0$ then combines the outputs $Z_k^{t, m}$ from all modality encoders to accomplish the learning objective. Based on these definitions, the loss function for the aggregated training dataset at round $t$ is formulated as follows: 
\begin{align}
F_t(\Theta, \mathcal{D}^t)  = \frac{1}{K}\sum_{k=1}^K f_t\left (\theta^0 \left (Z_k^{t, 1}, \dots, Z_k^{t, M}  \right ),   Y_k^{t}   \right)
\end{align}
Here, $\Theta = \left\{\theta^0, \theta^1, \dots, \theta^M \right\}$ denotes the overall model. The term $\theta^0 \left (Z_k^{t, 1}, \dots, Z_k^{t, M}  \right )$ represents the predicted labels generated by the head encoder, while $f_t$ is the loss function measuring the difference between the predictions and the ground-truth labels. Notably, only samples containing all modalities can be used to update the client’s loss function. This constraint can significantly hinder learning performance, as a large portion of the additional data, available only in a subset of modalities, is excluded from the model update process.

Since the training process relies on continuously collected real-time data rather than a fixed dataset, adopting an online learning approach is essential. Let the global model at each round be denoted as $\Theta^1, \ldots, \Theta^T$. The learning regret, $\text{Reg}_T$  is defined to measure the difference between the cumulative loss incurred by the learner and that of the best fixed model chosen in hindsight. Specifically:
\begin{align}
  \text{Reg}_T = \sum_{t=1}^{T}  F_t (\Theta^t; \mathcal{D}^t ) - \sum_{t=1}^{T} F_t (\Theta^*; \mathcal{D}^t) \label{regret}
\end{align}
Here, $\Theta^* = \arg\min_\Theta \sum_{t=1}^{T} F_t (\Theta; \mathcal{D}^t)$  denotes the optimal fixed model chosen in hindsight. The goal is to minimize the learning regret, which is equivalent to minimizing the cumulative loss over all rounds. Notably, if the regret grows sublinearly with respect to $T$, the online learning algorithm can asymptotically reduce the training loss over time.

The fixed optimal strategy in hindsight refers to a policy that could be devised by a centralized entity with complete prior knowledge of all per-round loss functions. In our setting, achieving such a strategy would require access to future information, including all upcoming data collections across rounds, which is inherently unpredictable due to its stochastic nature. Consequently, the complete loss functions are unknown at the outset and evolve dynamically over time. In this context, regret serves as a theoretical measure to quantify the performance gap between the proposed algorithm and the ideal but unattainable optimal strategy. For empirical evaluation, the performance of the proposed online algorithm is assessed using practical metrics such as test accuracy.

\section{The Workflow of MMO-FL with Modality Quantity and Quality Imbalance }
This section outlines the overall workflow of MMO-FL and examines the challenges posed by modality quantity and quality imbalance during the learning process. MMO-FL extends the HFL framework, in which each client independently collects local data samples and the server aggregates these updates to form the global model. Under modality imbalance, however, MMO-FL faces both quantitative and qualitative disparities among different modalities, which can disrupt the training process and substantially impair learning performance. In each global round $t \in \mathcal{T}$, where $\mathcal{T} = \{ 0, 1, 2, \dots, T-1 \}$, all clients perform a specified number of local training iterations, denoted by the parameter $E$. The index $\tau = 0, 1, 2, ..., E$ tracks these local iterations. Each global round $t$ then proceeds through the following steps. 

\subsubsection{\textbf{Client - New Data Collect}} At the start of each global round $t$, client $k$ aims to collect a new training dataset $ \mathcal{D}_k^t$ for that round. However, due to the inherent instability of IoT sensors, the collected data may exhibit modality quantity and quality imbalances. Data collection is conducted separately for each modality, producing paired feature–label sub-datasets that capture both quantitative and qualitative characteristics. To differentiate between the effects of quantity imbalance and quality imbalance, we define two types of local datasets to represent these conditions. The dataset only affected by \textbf{modality quantity imbalance} is denoted as $\mathcal{D}_k^{t,-}$ and is defined as:
\begin{align}
\mathcal{D}_k^{t,-} = \left (  \left [ \left ( X_k^{t, 1}, \hat{P}_k^{t, 1}\right )  \dots, \left  ( X_k^{t, M}, \hat{P}_k^{t, M}\right ); {Q}_k^{t}; Y_k^{t} \right ]\right ) \notag
\end{align}
Here, $\hat{P}_k^{t, m}$ represents the actual data availability status for modality $m$, which differs from the ideal collection state ${P}_k^{t, m}$. In the ideal case, all entries of ${P}_k^{t, m}$ are equal to 1. In contrast, $\hat{P}_k^{t, m}$ may contain zeros, indicating missing modality data for certain samples. To further characterize the impact of \textbf{modality quality imbalance}, we further define the dataset $\tilde{\mathcal{D}}_k^{t,-}$ as follow:
\begin{align}
\tilde{\mathcal{D}}_k^{t,-} = \left (  \left [ \left ( \tilde{X}_k^{t, 1}, \hat{P}_k^{t, 1}\right )  \dots, \left  ( \tilde{X}_k^{t, M}, \hat{P}_k^{t, M} \right ); \hat{Q}_k^{t}; Y_k^{t} \right ]\right ) \notag
\end{align}
In this context, $\hat{Q}_k^{t, m}$ denotes the actual quality status of the data collected for modality $m$ by client $k$ at global $t$. A value of 1 indicates that the collected data is of normal quality, i.e., $\tilde{X}_k^{t, m} = {X}_k^{t, m}$, whereas a value of 0 indicates that the data is of low quality, represented as $\tilde{X}_k^{t, m}$. Such online fluctuations in data quantity and quality do not occur in conventional offline multimodal FL settings. Although training can still proceed with the current imbalanced dataset, these conditions inevitably affect the training process and learning performance. Therefore, dedicated strategies are necessary to handle such imbalances effectively and ensure robust model training.

\subsubsection{\textbf{Client - Local Model Update}}
After data collection, each client $k$ initializes its local training using the current global model $\Theta^{t}$, received from the server. This model serves as the starting point for updating a new local model based on the collected, real-world, non-ideal training dataset $\tilde{\mathcal{D}}_k^{t,-}$. Each client performs $E$ iterations of online gradient descent (OGD) \cite{wang2024online} using the full training dataset. The update process is formulated as follows:
 \begin{align}
        &\Theta_k^{t,0} = \Theta^{t} \notag\\
        &\tilde{\Theta}_k^{t, \tau + 1} = \tilde{\Theta}_k^{t, \tau } - \eta \tilde{\textbf{G}}_{k}^{t, \tau, -}, \quad \forall \tau = 1, ..., E \notag\\
        &\tilde{\Theta}_k^{t+1} = \tilde{\Theta}_k^{t, E} 
\end{align}
Here, $\tilde{\textbf{G}}_{k}^{t, \tau, -} = \nabla F_{t}(\tilde{\Theta}_k^{t, \tau }, \tilde{\mathcal{D}}_k^{t,-})$ denotes the gradient computed on the current local dataset $\tilde{\mathcal{D}}_k^{t,-}$, which contains all modalities but is subject to both quantity and quality imbalance. The term ${\textbf{G}}_{k}^{t, \tau, -} = \nabla F_{t}(\tilde{\Theta}_k^{t, \tau }, {\mathcal{D}}_k^{t,-})$ represents the gradient influenced only by modality quantity imbalance, while $\tilde{\textbf{G}}_{k}^{t, \tau} = \nabla F_{t}(\tilde{\Theta}_k^{t, \tau }, \tilde{\mathcal{D}}_k^{t})$ corresponds to the gradient affected solely by modality quality imbalance. In all cases, $\eta$ denotes the learning rate. 

It should be noted that modality quantity imbalance causes some samples to contain complete modality information, while others include only partial modality data. At the same time, modality quality imbalance introduces discrepancies between certain modality representations and their true underlying values. Although these imbalances may degrade model performance, the training process remains viable, and model updates can still be carried out.

\subsubsection{\textbf{Client - Local Model Upload}}
Each client will upload the corresponding local model $\tilde{\Theta}_k^{t+1}$ to the server after finishing the $E$ iterations of local model update.

\subsubsection{\textbf{Server - Global Model Update}}
The server updates the global model by aggregating the local model updates received from the clients, as expressed by the following equation:
\begin{align}
    \Theta^{t+1} = \frac{1}{K}\sum_{k \in \mathcal{K}} \tilde{\Theta}^{t+1}_k 
\end{align}
The server then transmits the updated global model to all clients for the next training round. This process is repeated until the predefined number of global rounds is completed in the online learning setting.

From the preceding MMO-FL workflow under modality quantity and quality imbalance, it is clear that training can still proceed using the non-ideal local dataset $\tilde{\mathcal{D}}_k^{t,-}$. However, the resulting performance is likely to be suboptimal due to missing modalities and degraded data quality. A natural approach is to augment or correct $\tilde{\mathcal{D}}_k^{t,-}$,  but this is often impractical because complete modality information is rarely available. To address this challenge, we propose an online compensation mechanism that operates concurrently with the training process to mitigate the negative effects of modality imbalance. Before detailing the algorithm, we present a comprehensive theoretical analysis quantifying the impact of modality quantity and quality imbalance on the regret bound within the MMO-FL framework.

\section{Theoretical Analysis }
In this section, we present a detailed regret analysis of the proposed MMO-FL algorithm, highlighting several additional challenges in the theoretical study. First, the online learning paradigm requires assessing the long-term cumulative loss ${Reg}_T$ rather than focusing solely on convergence at a specific round. Second, due to the multimodal nature of the problem, the feature models corresponding to each modality must be examined individually, making it unsuitable to treat the local model as a single unified entity. Third, appropriate assumptions and configurations are necessary to account for both modality quantity and quality imbalance. These factors call for new proof techniques and supplementary assumptions to accurately model the problem’s complexity. The proof is conducted in three stages:  (1) We derive the regret bound for the case where the number of local iterations satisfies $ E > 1$ without considering modality quantity and quality imbalance. (2) We extend the analysis to account for modality quality imbalance under the same local iteration setting. (3) We will further explore the regret bound when $ E > 1$ while accounting for the impact of both modality quantity and quality imbalance.

\subsection{MMO-FL with local iterations $E>1$ and without modality quantity and quality imbalance}
To facilitate our analysis, we first introduce several additional definitions. After applying some basic transformations to the global model update equation above, we derive an alternative form of the global model update equation for the case where the local iteration is $E>1$ and no modality quantity and quality imbalance happen, as follows:
\begin{align}
\Theta^{t+ 1, 0} = \Theta^{t, 0} - \frac{\eta}{K} \sum_{k=1}^K \sum_{\tau = 0}^{E -1} {\textbf{G}}^{t, \tau}_k
\end{align}
where $\textbf{G}^{t, \tau}_k$ denotes the gradient of the local overall model for client $k$ across all $M$ modalities for round $t$ and local iteration $\tau$, which equals:
\begin{align}
\textbf{G}^{t, \tau}_k &= \left [  {\left ( \textbf{G}^{t, \tau, 0}_k \right )}^{\top}, \dots  {\left ( \textbf{G}^{t, \tau, m}_k  \right )}^{\top},  \dots  {\left ( \textbf{G}^{t, \tau, M}_k  \right )}^{\top} \right ]^{\top}
\end{align}
Subsequently, we will introduce the assumptions that are standard for analyzing online convex optimization, as referenced in \cite{wang2023online, wang2025denoising}. Some assumptions are defined at the modality level, tailored to the requirements of the multimodal parameter formulation.

\begin{assumption}
For any $\mathcal{D}^t$, the loss function $F_t (\Theta; \mathcal{D}^t)$ is convex with respect to $\Theta$ and differentiable. 
\label{assm:cov}
\end{assumption}

\begin{assumption}
The loss function is $L$-Lipschitz continuous, the partial derivatives for each modality satisfies: $\left \| \nabla_{m}  F_t (\Theta)   \right \|^2 \leq L^2$. 
\label{assm:bpd}
\end{assumption}

\begin{assumption}
The partial derivatives, corresponding to the consistent loss function, fulfills the following condition:
\begin{align}
\left \|  {\textbf{G}}_{k}^{t, \tau'} - {\textbf{G}}^{t, \tau}_{k} \right \| \leq \varphi \left \| \Theta_k^{t, \tau'} - \Theta_k^{t, \tau} \right \| \notag
\end{align}
\label{assm:gradient-change}
where $\tau'$ and $\tau$ indicate that they correspond to different local iterations.
\end{assumption}

For the purposes of subsequent theoretical analysis, we consider a $D$-dimensional vector for each modality in both the overall gradient and model. We define an arbitrary vector element $d \in [1, D]$ in overall gradient for modality $m$ as ${\textbf{G}}_{k, d}^m$, and similarly, the arbitrary vector element $d \in [1, D]$ in overall model for modality $m$ is denoted as $\Theta_{k, d}^m$. Including the head encoder $m = 0$, each overall gradient and model consists of a total of $(M + 1)D$ vector elements.

\begin{assumption}
The arbitrary vector element $d$ in the overall model $\Theta_{k, d}^m$ for any modality $m $ is bounded by: $\left | \Theta_{k, d}^m \right | \leq \sigma  $.
\label{assm:model-variant}
\end{assumption}

Assumption \ref{assm:cov} guarantees the convexity of the objective function, allowing us to apply convex optimization properties. Assumption \ref{assm:bpd} places a bound on the magnitude of the loss function’s partial derivatives at the modality level. Assumption \ref{assm:gradient-change} restricts the variation in these partial derivatives to a specific range, corresponding to the changes in the model across two different local iterations under the same loss function, thereby reflecting the smoothness property. Finally, Assumption \ref{assm:model-variant} defines the allowable range for each vector element of the overall model. Then, we derive the following Theorem \ref{thm:general1}.

\begin{theorem}\label{thm:general1}
Under Assumption 1-4, MMO-FL with local iterations $E>1$ and excluding the impact of modality quantity and quality imbalance, achieves the following regret bound:
\begin{align}
 & {Reg}_T   = \sum_{t=1}^{T} \sum_{k=1}^K \mathbb{E}_t \left [ F_t (\Theta^{t,0}; \mathcal{D}^t_k)  \right ] - \sum_{t=1}^{T} \sum_{k=1}^K F_t (\Theta^*; \mathcal{D}^t_k) \notag \\
 & \leq \frac{K  \left \| \Theta^{1, 0} - \Theta^* \right \|^2}{2 \eta E}  + \frac{\eta T K E (M+1) L^2}{2} \notag  \\
 & + 2\eta T D E K (M+1)^2 \varphi \sigma L \notag
\end{align}
\end{theorem}
\begin{proof}
The proof can be found in Appendix B.
\end{proof}
According to Theorem \ref{thm:general1}, by setting $\eta = \mathcal{O}(1/\sqrt{T})$, the MMO-FL can achieve a sublinear regret rate $\mathcal{O}(\sqrt{T})$. A sublinear regret bound means that the average regret per round, calculated as the total regret divided by the number of rounds, approaches zero as the number of rounds increases indefinitely. In the following analysis, we will consider the effect of modality quality and quantity imbalance.

\subsection{MMO-FL with local iterations $E>1$ and with modality quality imbalance}
In this section, we extend the proof to the case where the number of local iterations satisfies $E>1$ while accounting for the impact of modality quality imbalance. Using a similar transformation approach as before, we obtain the global model update equation as follows:
\begin{align}
\Theta^{t+ 1, 0} = \Theta^{t, 0} - \frac{\eta}{K} \sum_{k=1}^K \sum_{\tau = 0}^{E -1} \tilde{\textbf{G}}^{t, \tau}_k
\end{align}

We introduce an additional assumption to bound the influence of low-quality data on gradient computation. In particular, we assume that the reduction in data quality results from noise introduced during the data collection process.

\begin{assumption}
For any arbitrary component $d$ of the overall gradient corresponding to modality $m$, the difference between the gradients computed with and without the influence of noise is bounded by a finite range: 
\begin{align}
\left |\tilde{\textbf{G}}^{t, \tau, m}_{k, d} - {\textbf{G}}^{t, \tau, m}_{k, d}  \right | \leq \rho_m
\end{align}
where $\rho_m$ denotes the maximum deviation induced by noise in modality $m$.
\label{assm:gradient-noise}
\end{assumption}

Assumption \ref{assm:gradient-noise} ensures that the deviation in each component of the gradient vector remains bounded within a specified range when comparing gradients computed with and without noise. The parameter $\rho$ quantifies the extent of the noise impact, where a larger value of $\rho$ indicates a higher degree of modality quality imbalance. Based on this assumption, we derive Theorem \ref{thm:general2}, which incorporates the effect of modality quality imbalance into the analysis.

\begin{theorem}\label{thm:general2}
Under Assumption 1-5, MMO-FL with local iterations $E>1$ and including the impact of modality quality imbalance, achieves the following regret bound:
\begin{align}
& {Reg}_T   = \sum_{t=1}^{T} \sum_{k=1}^K \mathbb{E}_t \left [ F_t (\Theta^{t,0}; \tilde{\mathcal{D}}^t_k)  \right ] - \sum_{t=1}^{T} \sum_{k=1}^K F_t (\Theta^*; \tilde{\mathcal{D}}^t_k) \notag \\
& \leq  \frac{ K \left \| \Theta^{1, 0} - \Theta^* \right \|^2}{2 \eta E }  +  2 T \eta  D E K (M+1)^2 \sigma \varphi L  \notag \\
& + 2 (M+1) T D K  \sigma \rho_{\max} + \eta T K E (M+1) L^2  \notag  \\
& +  \eta T K E (M+1)^2 D^2 \rho_{\max}^2 \notag
\end{align}
Where $\rho_{\max} = \max_{m \in \mathcal{M}}\rho_{m}$.
\end{theorem}
\begin{proof}
The proof can be found in Appendix D.
\end{proof}
Based on Theorem 2, by choosing the learning rate as $\eta = \mathcal{O}(1/\sqrt{T})$, MMO-FL with the impact of modality quality imbalance can achieve a regret bound of $\mathcal{O}(\sqrt{T} + T \rho_{\max})$ over $T$ time rounds. Our analysis reveals that the term $\mathcal{O}(T \rho_{\max})$, which accounts for the low quality, plays a crucial role in determining whether a sublinear regret bound can be achieved. If appropriate algorithms are designed such that $\rho_{\max}$ decreases with $T$, the overall regret bound remains sublinear. Next, we extend the analysis to incorporate the impact of modality quantity imbalance, which leads to the following case.

\subsection{MMO-FL with local iterations $E>1$ and with modality quantity and quality imbalance}
In this section, we extend the proof to the case where the local iteration $E>1$ and considering the effect with both modality quality and quantity imbalance. Following a similar transformations approach, we derive the global model update equation as follows:
\begin{align}
\Theta^{t+ 1, 0} = \Theta^{t, 0} - \frac{\eta}{K} \sum_{k=1}^K \sum_{\tau = 0}^{E -1} \tilde{\textbf{G}}^{t, \tau, -}_k
\end{align}
To further characterize the impact of modality quantity imbalance, we introduce an additional assumption that bounds the deviation between the gradient computed on individual data samples and the corresponding batch gradient.
\begin{assumption}
For any arbitrary component $d$ of the overall gradient corresponding to modality $m$,  the deviation between the gradient computed from a single data sample and the corresponding batch gradient is bounded within a finite range:
\begin{align}
\left |  g^{t, \tau, m}_{k, d} - {\textbf{G}}^{t, \tau, m}_{k, d}  \right | \leq  \gamma_m
\end{align}
Here, $ g^{t, \tau, m}_{k, d}$ denotes the gradient computed from a single data sample, while ${\textbf{G}}^{t, \tau, m}_{k, d} = \frac{1}{N} \sum_{n=1}^N  g^{t, \tau, m}_{k, d}$ represents the corresponding batch gradient averaged over $N$ samples.
\label{assm:gradient-one}
\end{assumption}
Assumption \ref{assm:gradient-one} ensures a bounded deviation between the gradient computed from individual data samples and the average gradient over the batch. This allows us to quantify the impact of modality quantity imbalance. Then we derive Theorem \ref{thm:general3}, which integrates the effects of both modality quality and quantity imbalance into the regret bound.

\begin{theorem}\label{thm:general3}
Under Assumption 1-6, MMO-FL with local iterations $E>1$ and including the impact of modality quantity and quality imbalance, achieves the following regret bound:
\begin{align}
& {Reg}_T   = \sum_{t=1}^{T} \sum_{k=1}^K \mathbb{E}_t \left [ F_t (\Theta^{t,0}; \tilde{\mathcal{D}}^{t, -}_k)  \right ] - \sum_{t=1}^{T} \sum_{k=1}^K F_t (\Theta^*; \tilde{\mathcal{D}}^{t, -}_k) \notag \\
& \leq  \frac{ K \left \| \Theta^{1, 0} - \Theta^* \right \|^2}{2 \eta E } + \frac{3}{2} \eta T K E (M+1) L^2 \notag \\
& + 6 \eta T K E (M+1)^2 D^2  \gamma_{\max}^2  \left (\frac{N-N_{\min}}{N}\right )^2 \notag\\
& + \frac{3}{2} \eta T K E (M+1)^2 D^2  \rho_{\max}^2 + 2 \eta T D E K (M+1)^2 \sigma \varphi L \notag\\
&  + 2 TDK (M+1) \sigma \gamma_{\max} \left (\frac{N-N_{\min}}{N}\right ) \notag \\
& + 4 TDK (M+1) \sigma \rho_{\max}
\end{align}
Where $\gamma_{\max} = \max_{m \in \mathcal{M}} \gamma_{m}$ and $N_{\min} = \min_{m, t, k}N_k^{t, m}$.
\end{theorem}
\begin{proof}
The proof can be found in Appendix F.
\end{proof}

Based on Theorem \ref{thm:general3}, by choosing the learning rate as $\eta = \mathcal{O}(1/\sqrt{T})$, MMO-FL with the impact of modality quality and quantity imbalance can achieve a regret bound of $\mathcal{O}(\sqrt{T} + T \rho_{\max} + T \gamma_{\max} \left (\frac{N-N_{\min}}{N}\right ) )$ over $T$ time rounds. In addition to the term $\mathcal{O}(T \rho_{\max} )$ that appeared in Theorem \ref{thm:general2}, the regret bound is further influenced by the term $T \gamma_{\max} \left (\frac{N-N_{\min}}{N}\right ) )$, which captures the effect of modality quantity imbalance. To achieve a sublinear regret bound, two strategies can be considered: reducing $\gamma_{\max}$ or increasing $N_{\min}$ to approach $N$. Since $\gamma_{\max}$ is typically determined by the data and not controllable, a more practical approach is to compensate for missing data to raise $N_{\min}$ toward $N$. This observation motivates the design of the algorithm proposed in the subsequent section. The theoretical analysis above provides clear guidance for algorithm design. In the following section, we introduce the rebalancing algorithm, which leverages prototype learning to address both types of modality imbalance.

\section{Modal Quantity and Quality Rebalancing Algorithm } 
Building upon the theoretical analysis presented in the previous section, it is evident that both modality quality imbalance and modality quantity imbalance have a substantial impact on the regret bound and ultimately degrade learning performance. To mitigate these effects, dedicated algorithmic solutions are necessary. In this section, we introduce the Modality Quantity and Quality Rebalancing (QQR) algorithm, which leverages prototype learning techniques \cite{yang2023fedhap} to enable effective and efficient rebalancing of modality quality and quantity during the learning process. The proposed QQR algorithm is detailed in Algorithm~\ref{alg: qqr} and comprises three main components: Online Global Prototype Construction (OGPC), Prototypical Quantity Rebalancing (PNR), and Prototypical Quality Rebalancing (PLR). The following subsections provide a detailed explanation of each component. 

\begin{algorithm}
    \caption{QQR Algorithm} \label{alg: qqr} 
    \begin{algorithmic}[1]
        \For {$t = 1, 2, ..., T - 1$}
            \State \textbf{Client Side:}
            \For {$k = 1, 2, ..., K$}
                \If { Client satisfies the condition $\Phi_{g}$}
                    \State Calculates the local prototype $v_{k, c}^{t, m}$  via Eq.~\ref{local_pro}.
                    \State Integrates and uploads the $v_{k, c}^{t, m}$ to server.
                \EndIf
                \State Downloads and utilizes the collection $\bar{\mathcal{V}}^t$.
                \State \textbf{Prototypical Quantity Rebalancing:}
                \State Builds the $\tilde{Z}_k^{t, m}$ via Eq.~\ref{pnr1} and Eq.~\ref{pnr2}.
                \State \textbf{Prototypical Quality Rebalancing:}
                \State Utilizes the PLR loss $\mathcal{L}_{PLR}$ in Eq.~\ref{losspce} for updating.
            \EndFor
            \State \textbf{Server Side:}
            \State Collects the local prototype $v_{k, c}^{t, m}$  from clients.
            \State Calculates the $v_{c}^{t, m}$ via Eq.~\ref{tgp}.
            \State Updates the $\bar{v}_{c}^{t, m}$ via Eq.~\ref{pgp}.
            \State Aggregates and distributes the $\bar{\mathcal{V}}^t$ to all clients.
        \EndFor
    \end{algorithmic}
\end{algorithm}
\vspace{-5pt}

\begin{figure*}[htbp]
    \centering
    \includegraphics[width=0.8\textwidth]{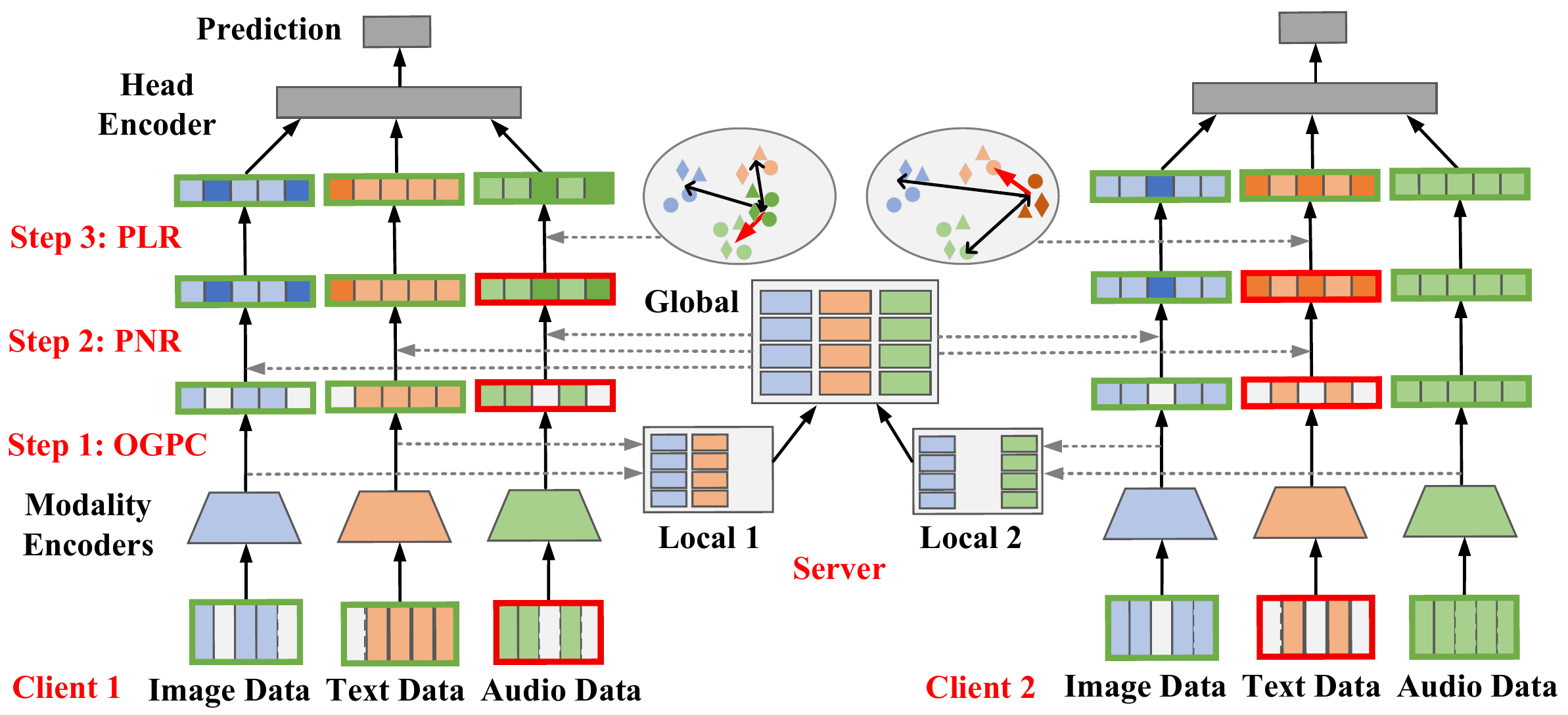} 
    \caption{Illustration of QQR: The QQR framework comprises three main components. First, OGPC generates local prototypes at each client and aggregates them at server to produce cumulative global prototypes. Second, PNR leverages the cumulative global prototype collection to substitute missing feature embeddings, thereby addressing modality quantity imbalance. Finally, PLR utilizes both the cumulative global prototype collection and the relative relationships among modalities to correct modality quality imbalance. The PNR and PLR procedures can be executed either independently or concurrently.}
    \vspace{-10pt}
    \label{QQR}
\end{figure*}

\subsection{Online Global Prototype Construction}
The global prototype construction process is designed to maintain long-term global prototypes, which are made available to clients as needed during the quantity and quality rebalancing stages. To progressively approximate the true class representations, these prototypes are updated at each global round $t$. Let $c \in  \left\{1, \dots, C \right\}$ denote the class labels. The local prototype is defined as the average of the feature representations extracted by the modality encoder from the available samples of a given class that meet the normal quality criterion. This can be formally expressed as follows:
\begin{align}
 &v_{k, c}^{t, m} = \frac{1}{ \sum_c  P_{k, c}^{t, m} } \sum_{n \in \mathcal{N}(\Phi_{g})}  \theta^m(x_{k, n}^{t, m} ) \label{local_pro}  \\
 &\textit{where} \quad \Phi_{g} = \left\{ p_{k, n}^{t, m}= 1, q_{k}^{t, m} =1,  y_{k, n}^t = c\right\} 
\end{align}
To ensure fairness among prototypes for different modalities, we introduce two additional principles. First, local prototypes derived from various modalities should conform to a unified structure. This ensures that the features extracted by the modality encoders, denoted as $\theta^m(x_{k, n}^{t, m} )$, remain consistent across modalities, even if the original data formats differ. By passing data through their respective encoders, the resulting features can be structurally aligned in a shared representation space. Second, to maintain comparability and training stability, local prototypes across modalities should be normalized, ensuring their magnitudes are consistent throughout the learning process.

At the end of each training round, every client generates local prototypes for each modality and transmits them to the server. Upon receiving these local prototypes, the server constructs an instantaneous global prototype for modality $m$ and class $c$ at global round $t$ using the following formulation:
\begin{align}
v_{c}^{t, m} = \frac{1}{K} \sum_{k \in \mathcal{K}}  v_{k, c}^{t, m} \label{tgp}
\end{align}
Given the nature of online learning, data collected at each global round may be biased and may not cover all class categories. According to the law of large numbers, as the number of samples increases, the sample mean tends to converge to the true population mean. To address this issue, it is essential to maintain a cumulative global prototype that continuously captures the semantic representation of each class for every modality across training rounds. The cumulative global prototype for modality $m$ and class $c$ at global round $t$ is defined as follows:
\begin{align}
\bar{v}_{c}^{t, m} = \frac{(t-1)\bar{v}_{c}^{t-1, m} + {v}_{c}^{t, m}}{t} \label{pgp} 
\end{align}
The cumulative global prototype $\bar{v}_{c}^{t, m}$ is continuously updated and maintained on the server, this will serve as a key component for subsequent quantity and quality rebalancing procedures. The server retains the full set of cumulative global prototypes, organized by modality and class, denoted as $\bar{\mathcal{V}}^{t}$, and defined as follows:
\begin{align}
\bar{\mathcal{V}}^{t}  =  
\begin{bmatrix}
\bar{v}_{1}^{t, 1} & \dots & \bar{v}_{c}^{t, 1} & \dots & \bar{v}_{C}^{t, 1} \\
\dots & \dots & \dots & \dots & \dots  \\
\bar{v}_{1}^{t, m} & \dots & \bar{v}_{c}^{t, m} & \dots & \bar{v}_{C}^{t, m} \\
\dots & \dots & \dots & \dots & \dots  \\
\bar{v}_{1}^{t, M} & \dots & \bar{v}_{c}^{t, M} & \dots & \bar{v}_{C}^{t, M} \\
\end{bmatrix}
\end{align}

This collection is continuously updated over successive training rounds, maintained on the server, and shared with clients as needed. In the following sections, we describe how clients utilize the cumulative global prototype collection to perform modality quantity rebalancing and modality quality rebalancing, respectively.

\subsection{Prototypical Quantity Rebalancing }
In this phase, each client executes the PNR algorithm. The core objective of PNR is to address modality quantity imbalance by compensating for discrepancies in sample counts across modalities through prototype-based augmentation. The overall process of the PNR algorithm is illustrated in Step 2 of Fig.~\ref{QQR}. Upon receiving the cumulative global prototype collection from the server, each client identifies its own modality imbalance scenario and generates virtual feature representations to substitute for missing features in missing data samples. This process is described in detail below:
\begin{align}
\tilde{Z}_k^{t, m} = \left [ \tilde{z}_{k, 1}^{t, m}, \dots, \tilde{z}_{k, n}^{t, m}, \dots, \tilde{z}_{k, N}^{t, m}\right ]  \label{pnr1}
\end{align}
where $\tilde{z}_{k, 1}^{t, m}$ denotes the corresponding feature representations for data sample $n$ of modality $m$ at client $k$ during round $t$. The value of $\tilde{z}_{k, n}^{t, m}$  is determined based on the availability status of the sample, as outlined below:
\begin{align}
 \tilde{z}_{k, n}^{t, m} = \left\{\begin{matrix}
 \theta^m(x_{k, n}^{t, m} ) &  \textit{if} \quad p_{k, n}^{t, m}= 1 \\
 \bar{v}_{c(k, n)}^{t, m} &  \textit{if} \quad p_{k, n}^{t, m}= 0 \\
\end{matrix}\right. \label{pnr2}
\end{align}
In this context, $c(k, n)$ denotes the class label associated with the $n$-th data sample of client $k$. The predicted label is then obtained by applying the head encoder $\theta^0$ to the aggregated feature representations across all modalities: 
\begin{align}
\theta^0 \left (Z_k^{t, 1}, \dots, \tilde{Z}_k^{t, m}, \dots Z_k^{t, M}  \right )
\end{align}
After this step, the standard MMO-FL training procedure, as outlined in the previous section, is carried out. This approach effectively mitigates the impact of modality quantity imbalance by aligning with the theoretical objective of increasing $N_{min}$, as discussed in our earlier analysis.

\subsection{Prototypical Quality Rebalancing}
In this phase, the client will execute the PLR algorithm. The core idea of the PLR algorithm is to alleviate modality quality imbalance by refining the representations of lower-quality modalities. This is achieved by aligning these representations toward their expected semantic positions, guided by the relative distances among modality-specific prototypes. This prototype-based alignment facilitates effective quality correction across modalities. The overall mechanism of PLR is illustrated in Step 3 of Fig.~\ref{QQR}. Inspired by prior work \cite{fan2023pmr}, we introduce the Prototype Cross Entropy (PCE), a loss function that captures the discrepancy between the feature representation of a given modality and its corresponding prototype in each global round. The PCE is formally defined as follows:
\begin{align}
\mathcal{L}_{PCE}^m (f_t) = \mathbb{E}  \left [ -\log \frac{\exp (- d (\tilde{z}_{k, n}^{t, m}, \bar{v}_{c(k, n)}^{t, m}))}{\sum_{c = 1}^C \exp (- d (\tilde{z}_{k, n}^{t, m}, \bar{v}_{c}^{t, m}))} \right ]
\end{align}
Here, $d(\cdot, \cdot)$ denotes the distance function, which is instantiated as the Euclidean distance. The Prototypical Quality Rebalancing loss, denoted by $\mathcal{L}_{PLR}$, is defined as a weighted combination of the standard cross-entropy (CE) loss and PCE loss. It is given as follows:
\begin{align}
\mathcal{L}_{PLR} =  \mathcal{L}_{CE} + \beta (1 -q_{k}^{t, m}) \mathcal{L}_{PCE}^m  \label{losspce}
\end{align}
In this formulation, $\beta$ is a hyper-parameter to control the degree of modulation. The indicator variable $q_{k}^{t, m}$ represents the quality status of modality $m$ for client $k$ at round $t$. If the modality is identified as low quality, $q_{k}^{t, m} = 0$, the loss function incorporates the PCE term to reduce the impact of the degraded data. Otherwise, when $q_{k}^{t, m} = 1$, no adjustment is applied. Each client dynamically adjusts its training loss based on the quality status of its local modalities.

The steps described above form the core components of the proposed QQR algorithm. In the following, we present additional considerations and potential extensions aimed at enhancing the algorithm’s effectiveness. 

\textbf{Communication Efficiency Strategy:} Within the QQR algorithm, the cumulative global prototype functions as an approximate representation that captures the most salient and discriminative features of the original data. However, high numerical precision may not be necessary for this purpose. To further reduce communication costs, quantization techniques can be employed when transmitting the local prototypes. In our experimental evaluation, we assess the impact of this strategy on learning performance under both modality quantity and quality imbalance scenarios.

\section{Experiment}
In this section, we will experimentally evaluate the performance of the MMO-FL algorithm under conditions of modality quantity and quality imbalance. The experiments were conducted on an Ubuntu 18.04 PC equipped with an Intel Core i7-10700KF 3.8GHz CPU and a GeForce RTX 3070 GPU. The detailed experimental settings are described below.

\subsection{Datasets}
To simulate MMO-FL in IoT scenarios, we employ two real-world multimodal datasets: UCI-HAR and MVSA-Single. A detailed description of each dataset is provided below.

\textbf{UCI-HAR}: The UCI-HAR dataset is a well-established benchmark for human activity recognition. It contains a total of 10,299 samples collected from 30 individuals (average age 24), each performing one of six predefined activities: walking, walking upstairs, walking downstairs, sitting, standing, and lying down. Data were recorded using smartphone-based inertial sensors, including a 3-axis accelerometer and gyroscope, capturing motion signals in three dimensions. The sensors sampled at a rate of 50 Hz, producing 128 readings per axis within each time window. This dataset is used in our experiments to investigate sensor-based human activity recognition using 3D motion data. 

\textbf{MVSA-Single}: The MVSA-Single dataset is a benchmark resource widely used in multimodal sentiment analysis, emphasizing the joint modeling of textual and visual information from social media posts. It contains 5,129 samples, each comprising an image paired with accompanying text. Every image-text pair is annotated with one sentiment categories: Positive, Neutral, or Negative, representing the overall emotional expression conveyed through the combined modalities.

The original datasets are inherently static and tailored for conventional offline learning settings. To support the demands of online learning, it is necessary to convert them into dynamic datasets. The specific procedure used to achieve this transformation is outlined in detail in the following section.

\subsection{Online Data Generation}
In the context of online learning, the training dataset must evolve over time, with new data acquired at the beginning of each global round. To ensure that the model receives adequate data for effective training, an initial dataset is gathered at the outset of the learning process. Given the differences in size and characteristics between the UCI-HAR and MVSA-Single datasets, we adopt dataset-specific strategies for simulating online data generation in each case.

\textbf{UCI-HAR}: For the UCI-HAR dataset, the training setup includes five clients. Each client is initially allocated 2000 samples, drawn from a Dirichlet distribution with ratio $\alpha$ to simulate non-IID long-term data sources. During each global round, clients maintain an evolving local dataset consisting of 500 samples. To mimic online data collection, 20 new samples are added from the long-term pool at every round, while the oldest 20 samples are simultaneously discarded. This continuous replacement strategy ensures that local data remains dynamic.

\textbf{MVSA-Single}: For the MVSA-Single dataset, the training setup consists of five clients. Each client is initially provided with 1500 samples drawn using a Dirichlet distribution with ratio $\alpha$ to reflect a non-IID long-term data distribution. During training, each client maintains an evolving online dataset of 800 samples. At each global round, 20 new samples are added from the long-term data pool, and the oldest 20 samples are removed, ensuring a constant dataset size. This streaming update strategy supports the dynamic nature of online learning by continuously refreshing the local data.

\subsection{Imbalance Simulation}
\textbf{Quantity Imbalance Simulation}: To evaluate the impact of modality quantity imbalance, we introduce two parameters: $\lambda_p$, named intra-round quantity imbalance ratio, which denotes the proportion of missing data relative to the total dataset, and $\lambda_r$, named inter-round quantity imbalance ratio, which represents the proportion of global rounds affected by quantity imbalance relative to the total number of rounds. A higher $\lambda_p$ i indicates a greater extent of data loss during the collection process, while a higher  $\lambda_r$ reflects a more frequent occurrence of quantity imbalance throughout the training period. To isolate this effect and eliminate confounding factors, the experimental design includes the following constraints. First, we assume that all clients experience the same value of $\lambda_p$ across all global rounds. Second, given that the datasets used in the experiments consist of two modalities, we restrict each instance of missing data to affect only one modality per modality pair. In the following experiments, we systematically examine the impact of varying degrees of modality quantity imbalance on learning performance by controlling the value of $\lambda_p$ and $\lambda_r$.

\textbf{Quality Imbalance Simulation}: To evaluate the impact of modality quality imbalance, we introduce the parameter inter-round quality imbalance ratio $\delta_r$ to represent the proportion of rounds in which the client collects low-quality data out of the total number of learning rounds. A higher value of $\delta_r$ indicates that a client collects low-quality data in a greater number of global rounds. To control variables and isolate the effect of modality quality imbalance, the experimental design follows two rules. First, all clients are assigned the same $\delta_r$ value across training rounds, and second since the dataset contains only two modalities, at most one modality is permitted to collect low-quality data in any given round. In the subsequent experiments, we systematically investigate how varying degrees of modality quality imbalance affect learning performance by adjusting the value of $\delta_r$. 

\subsection{Model Details}
In the following, we detail the model architectures and key parameters used in our experiments, presented separately for the two datasets.

\textbf{UCI-HAR}: The dataset includes two distinct modalities: accelerometer and gyroscope signals, each requiring a dedicated encoder architecture. For the accelerometer data, we use a CNN-based model as the encoder, featuring five convolutional layers followed by a fully connected layer. For the gyroscope data, we use an LSTM-based model with one LSTM layer and one fully connected layer. Both encoders produce 128-dimensional feature vectors. A shared classification head, consisting of two fully connected layers, is used to generate the final predictions. The learning rate is 0.08, with a decay factor of 0.95 until it reaches 0.001. The batch size is  128.

\textbf{MVSA-Single}: The dataset includes two distinct modalities: text and
image data, which require separate encoder networks tailored to their data types. For the image data, a four-layer CNN is utilized, with its final layer adjusted to yield a 128-dimensional embedding. For text data, a two-layer LSTM network is employed, also configured to output a 128-dimensional feature vector. The modality-specific features are subsequently processed by a shared header model consisting of two fully connected layers. The learning rate initialized at 0.01, which exponentially decays by a factor of 0.99 until it reaches a minimum of 0.001. The batch size used throughout is 128.

\subsection{Benchmarks} 
n our experimental evaluation, we consider several baseline methods to benchmark the performance of our proposed approach. Since the challenges of modality quality and quantity imbalance in the context of MMO-FL remain unexplored, there are no well-established baselines available for direct comparison. To bridge this gap, we incorporate both optimistic and pessimistic bounds on learning performance. A detailed explanation of each benchmark is provided below.

To evaluate the effectiveness of the PNR algorithm, we consider the following three benchmarks for comparison.

\textbf{Full Collection (FC)}. In this scenario, each client is able to successfully collect complete modality information for all data samples, resulting in no quantity imbalance. Consequently, the entire dataset can be fully utilized for training, serving as an upper bound for achievable learning performance. 

\textbf{Incomplete Subset (IS)}. In this scenario, each client collects modality-specific data based on its actual operating conditions. Due to variations in sensor performance, some data samples may contain only partial modality information, leading to modality quantity imbalance. As a result, training can only be performed on the subset of data that contains all modalities after modality alignment, which serves as a lower bound for evaluating learning performance. 

\textbf{Zero Padding (ZP)}. In this scenario, to maximize data utilization based on the incomplete set, we apply zero-padding to compensate for the missing modality components. This allows training to be conducted using the available modality information from all collected data samples. 

To assess the effectiveness of the PLR algorithm, we compare it against the following two benchmarks.

\textbf{Pristine Quality (PQ)}. In this scenario, all modality-specific sensors on the client operate reliably, consistently collecting normal quality data without exhibiting temporal variations in data quality. 

\textbf{Bare Quality (BQ)}. In this scenario, certain modality-specific sensors on the client produce low-quality data due to hardware limitations or malfunctions. Furthermore, this degraded data is not corrected or enhanced through any compensation methods. Consequently, the client proceeds with local model training using the degraded input.

\subsection{Simulation Results}
In this section, we evaluate the proposed MMO-FL framework under modality quantity and quality imbalance scenarios. To enable controlled analysis and prevent variable cross-interference, we first examine the effect of quantity imbalance and assess the performance of the PNR algorithm. We then investigate the quality imbalance problem and evaluate the PLR algorithm. Next, we perform ablation studies to analyze performance under different parameter settings, considering the combined effects of quantity and quality imbalance. Finally, we assess the performance impact and communication efficiency achieved through quantized prototype uploads. All reported results are averaged over 10 random runs.

\begin{figure}[h]
\centering
\vspace{-15pt}
\subfloat[Test Accuracy with UCI-HAR]{\includegraphics[width=0.49\linewidth]{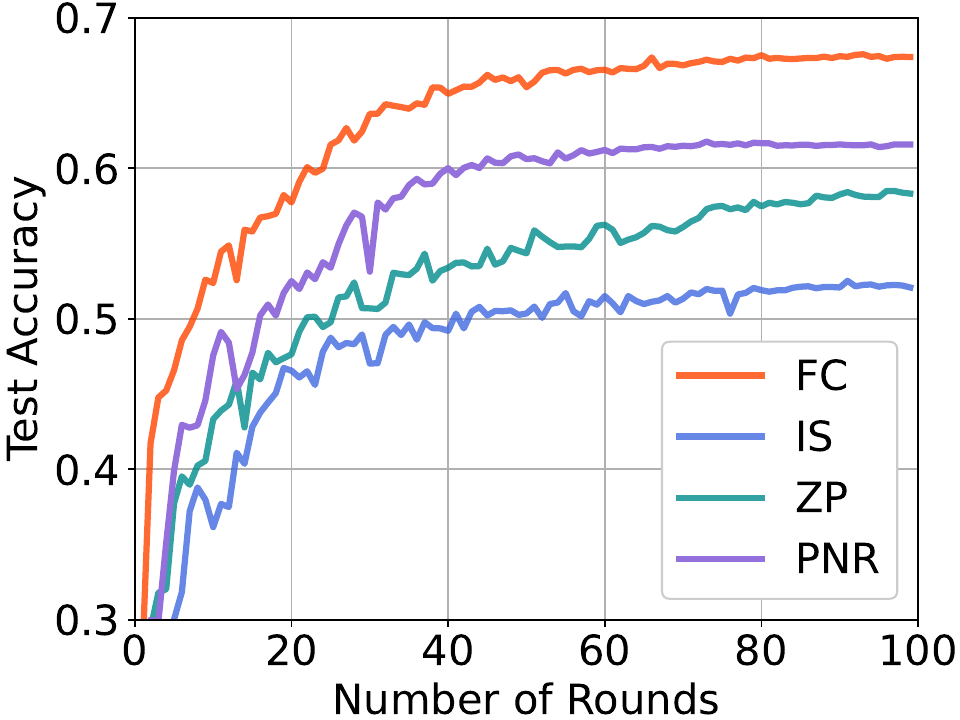}} 
\subfloat[Test Accuracy with MVSA-Single ]{\includegraphics[width=0.49\linewidth]{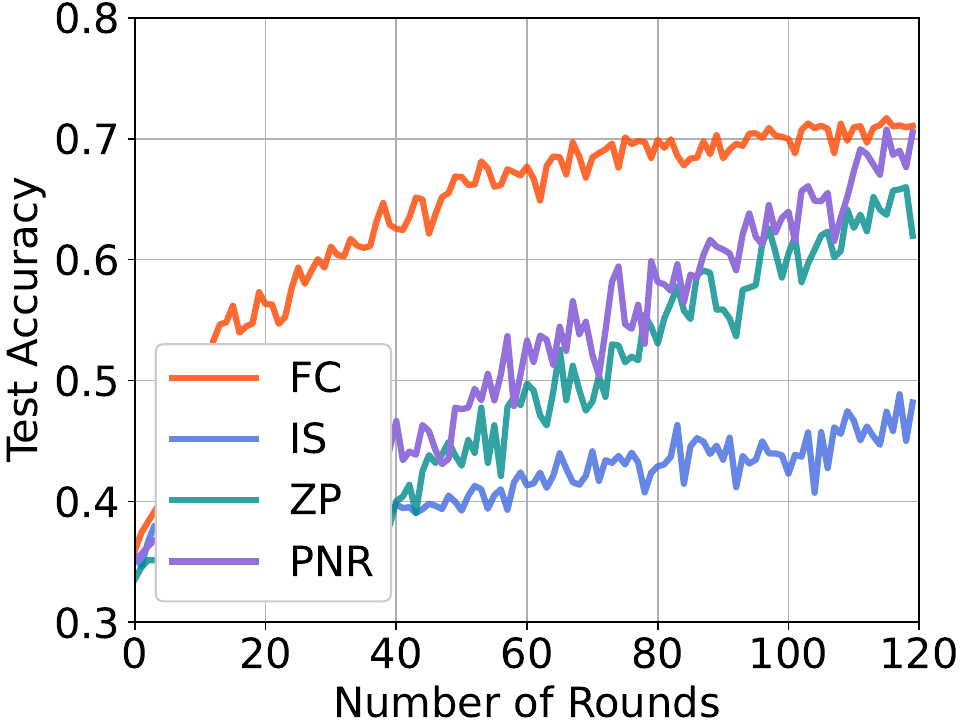}} 
\caption{Performance comparison of proposed algorithm and benchmarks with quantity imbalance.}   \label{mmofl-pnr}
\vspace{-10pt}
\end{figure}

\textbf{Performance Comparison (Quantity Imbalance)}. We begin by evaluating the learning performance with the impact of quantity imbalance within the MMO-FL scenario. The test accuracy comparison between proposed PNR and benchmarks are shown in Fig.~\ref{mmofl-pnr}(a) and Fig.~\ref{mmofl-pnr}(b) based on UCI-HAR dataset with configuration $[\lambda_r = 0.5, \lambda_p = 0.9, \alpha = 1]$ and MVSA-Single dataset with $[\lambda_r = 0.5, \lambda_p = 0.9, \alpha = 5]$, respectively. Based on the above figures, we get several key observations. First, the FC setting achieved the highest performance because it includes all data and does not experience quantity imbalance. In contrast, the IS setting yielded the lowest performance, indicating that quantity imbalance without mitigation can significantly degrade learning. Second, the ZP approach improved performance compared to IS, indicating that zero-padding missing inputs is more effective than completely omitting them during training. Finally, the proposed PNR algorithm outperformed ZP, as replacing missing data with cumulative global prototypes instead of zeros provides better inter-modal correspondence and results in superior learning outcomes. The above findings are consistent across both datasets. 

\begin{figure}[h]
\centering
\vspace{-15pt}
\subfloat[Test Accuracy with UCI-HAR ]{\includegraphics[width=0.49\linewidth]{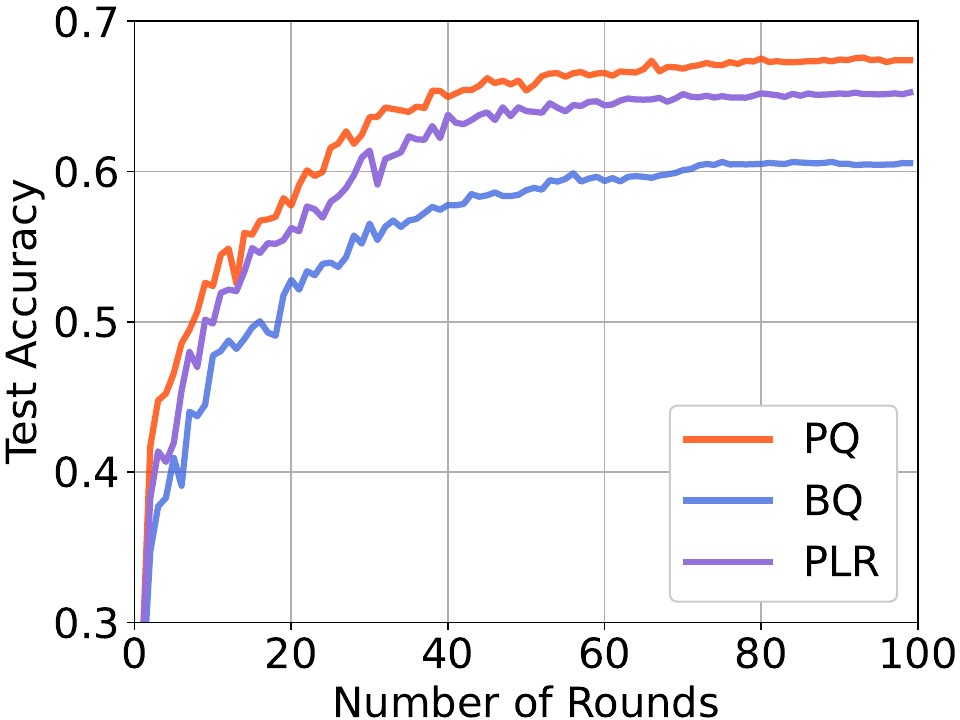}} 
\subfloat[Test Accuracy with MVSA-Single ]{\includegraphics[width=0.49\linewidth]{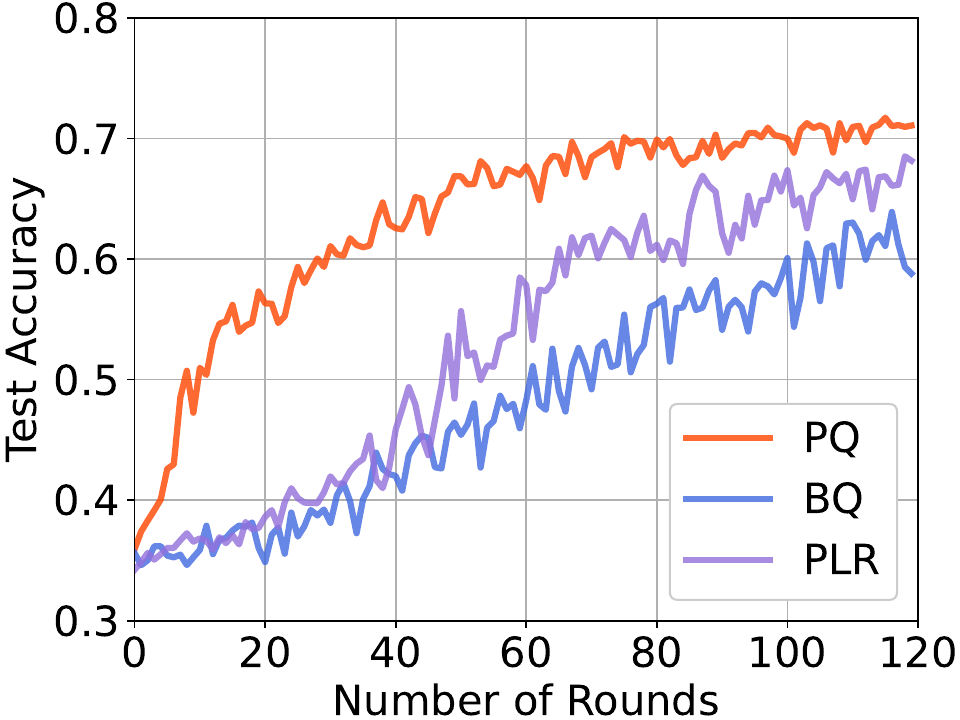}} 
\caption{Performance comparison of proposed algorithm and benchmarks with quality imbalance.}   \label{mmofl-plr}
\vspace{-15pt}
\end{figure}

\textbf{Performance Comparison (Quality Imbalance)}. In this part,  we will evaluate the learning performance with the impact of quality imbalance within the MMO-FL scenario. To emulate low-quality data conditions, we inject high-intensity additive white Gaussian noise (AWGN) with a SNR of 10 dB into the original dataset. The test accuracy comparison between proposed PLR and benchmarks are shown in Fig.~\ref{mmofl-plr}(a) and Fig.~\ref{mmofl-plr}(b) based on UCI-HAR dataset with configuration $[\delta_r = 0.5, \beta = 0.5, \alpha = 1]$ and MVSA-Single dataset with $[\delta_r = 0.5, \beta = 0.5, \alpha = 5]$, respectively. Based on the simulation results, we can summarize the following findings. First, a comparison between PQ and BQ shows that adding substantial noise to multimodal data produces low-quality inputs that markedly degrade learning performance. Furthermore, the proposed PLR algorithm effectively mitigates the negative impact of low-quality data, with performance improvements becoming more pronounced as the number of training rounds increases. This improvement is primarily attributed to the PCE loss, which reduces the distance between the feature extractor of low-quality data and cumulative global prototypes, thereby gradually diminishing the adverse effects of quality imbalance.

Based on the above performance comparisons, which confirm the effectiveness of the proposed PNR and PLR algorithms under the MMO-FL setting, we next conduct a series of ablation studies to examine the specific impact of key parameters on learning performance.

\begin{figure}[h]
\centering
\vspace{-15pt}
\subfloat[Test Accuracy with UCI-HAR ]{\includegraphics[width=0.49\linewidth]{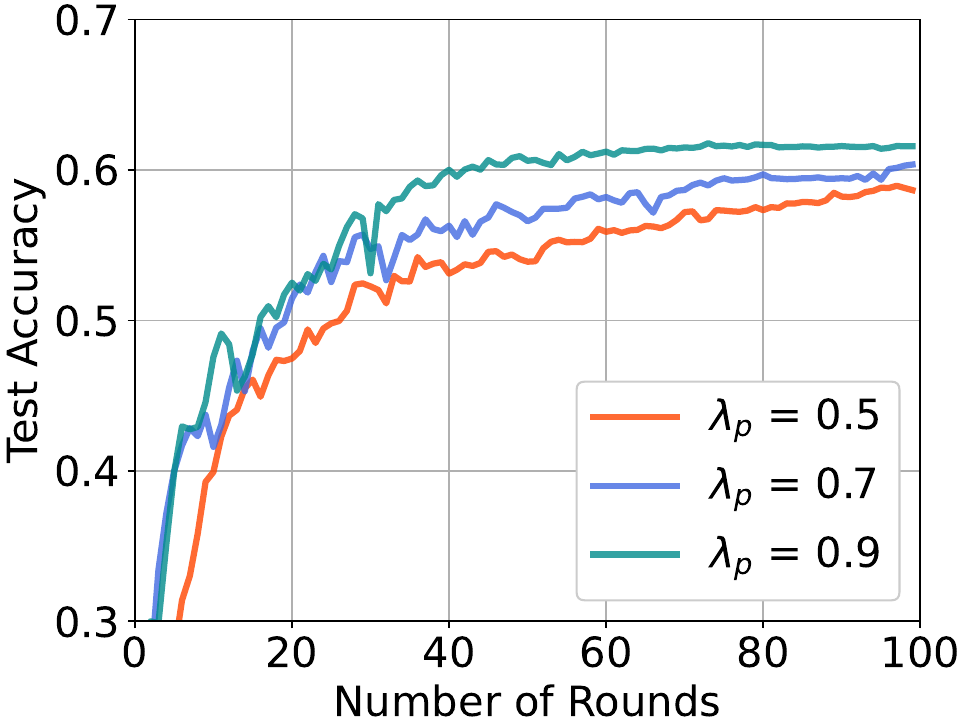}} 
\subfloat[Test Accuracy with MVSA-Single ]{\includegraphics[width=0.49\linewidth]{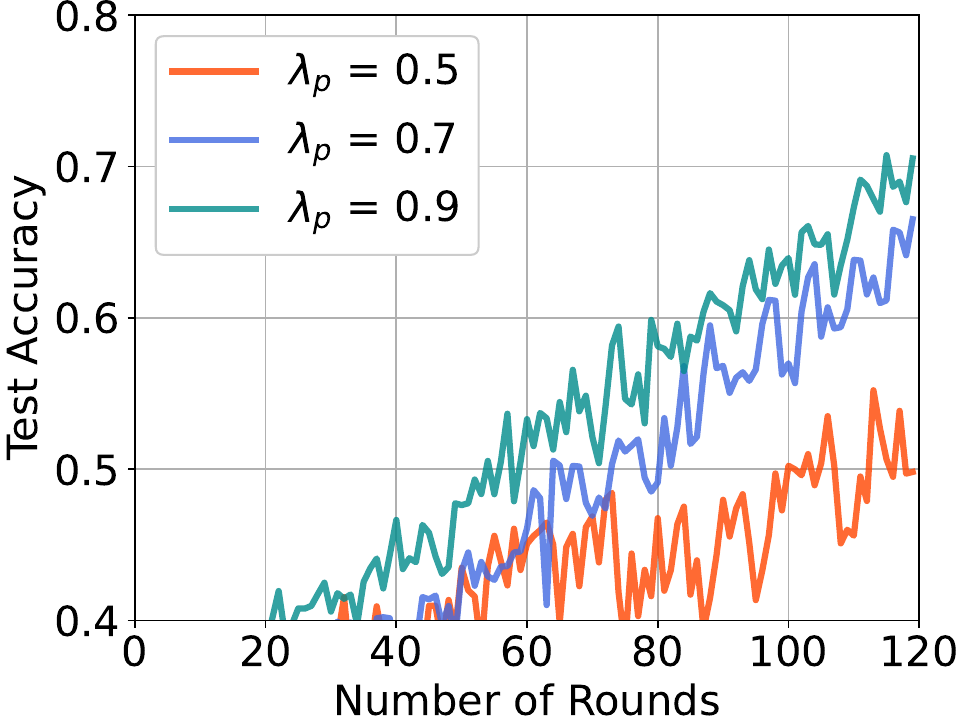}} 
\caption{Performance comparison of proposed algorithm and benchmarks with different $\lambda_p$ ratio.}   \label{mmofl-intra-n}
\vspace{-10pt}
\end{figure}

\textbf{Impact of Ratio $\lambda_p$ (Quantity Imbalance)}. In this part, we will explore the effects with the intra-round quantity imbalance ratio $\lambda_p$ on learning performance. The results for the UCI-HAR and MVSA-Single datasets are presented in Fig.~\ref{mmofl-intra-n}(a) and Fig.~\ref{mmofl-intra-n}(b), respectively. Several key observations can be made: First, as $\lambda_p$ increases, learning performance improves because more data is available and fewer quantity imbalances need to be compensated by the PNR algorithm. In contrast, when $\lambda_p$ decreases substantially (e.g., $\lambda_p = 0.5$), the MVSA-Single dataset shows a pronounced drop in performance. This decline results from severe data loss, where even with compensation, achieving the desired results becomes difficult. The degree of degradation varies across datasets.

\begin{figure}[h]
\centering
\vspace{-15pt}
\subfloat[Test Accuracy with UCI-HAR ]{\includegraphics[width=0.49\linewidth]{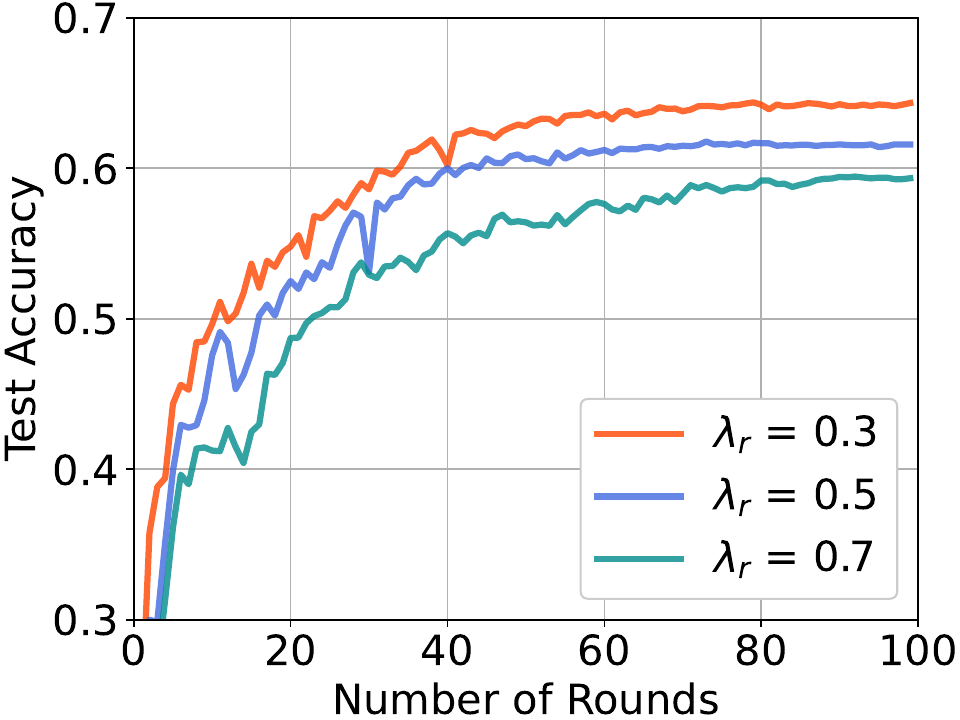}} 
\subfloat[Test Accuracy with MVSA-Single ]{\includegraphics[width=0.49\linewidth]{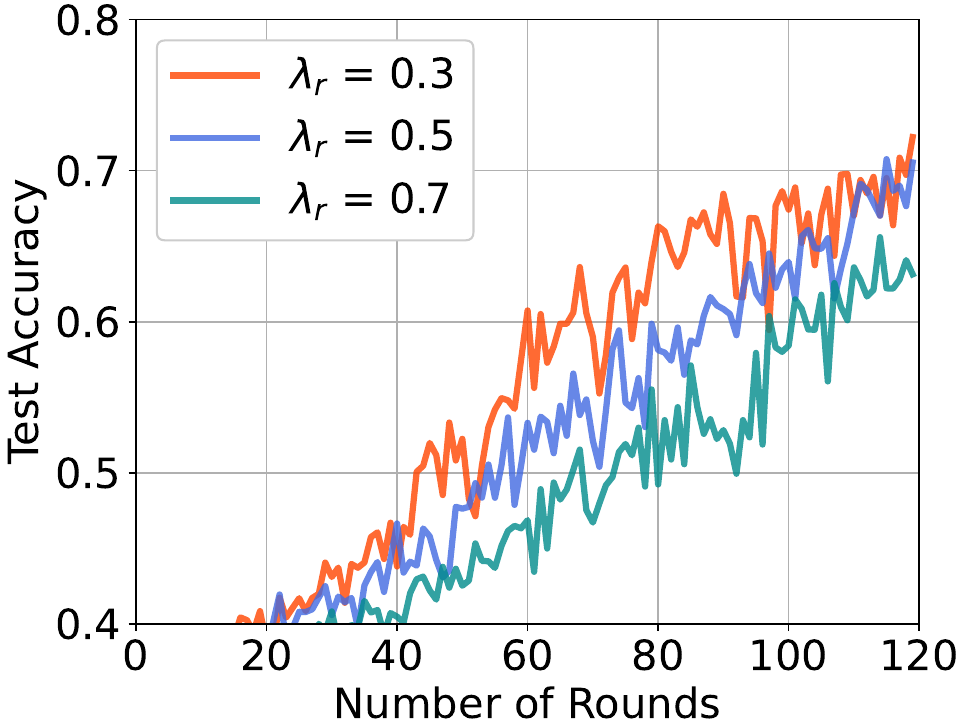}} 
\caption{Performance comparison of proposed algorithm and benchmarks with different $\lambda_r$ ratio.}   \label{mmofl-inter-n}
\vspace{-15pt}
\end{figure}

\textbf{Impact of Ratio $\lambda_r$ (Quantity Imbalance)}. In this part, we will explore the effects with the inter-round quantity imbalance ratio $\lambda_r$ on learning performance. The results for the UCI-HAR and MVSA-Single datasets are presented in Fig.~\ref{mmofl-inter-n}(a) and Fig.~\ref{mmofl-inter-n}(b), respectively. We have the following findings: First, smaller values of $\lambda_r$ yield better learning performance because the proportion of rounds with quantity imbalance is lower, allowing more rounds to collect complete modal data. When $\lambda_r$ reaches 0.7, both datasets exhibit a significant drop in learning performance. This decline is due to the high frequency of quantity imbalance, where even with compensation, sustaining strong learning performance remains difficult. This phenomenon is consistently observed across both datasets.

\begin{figure}[h]
\centering
\vspace{-15pt}
\subfloat[Test Accuracy with UCI-HAR]{\includegraphics[width=0.49\linewidth]{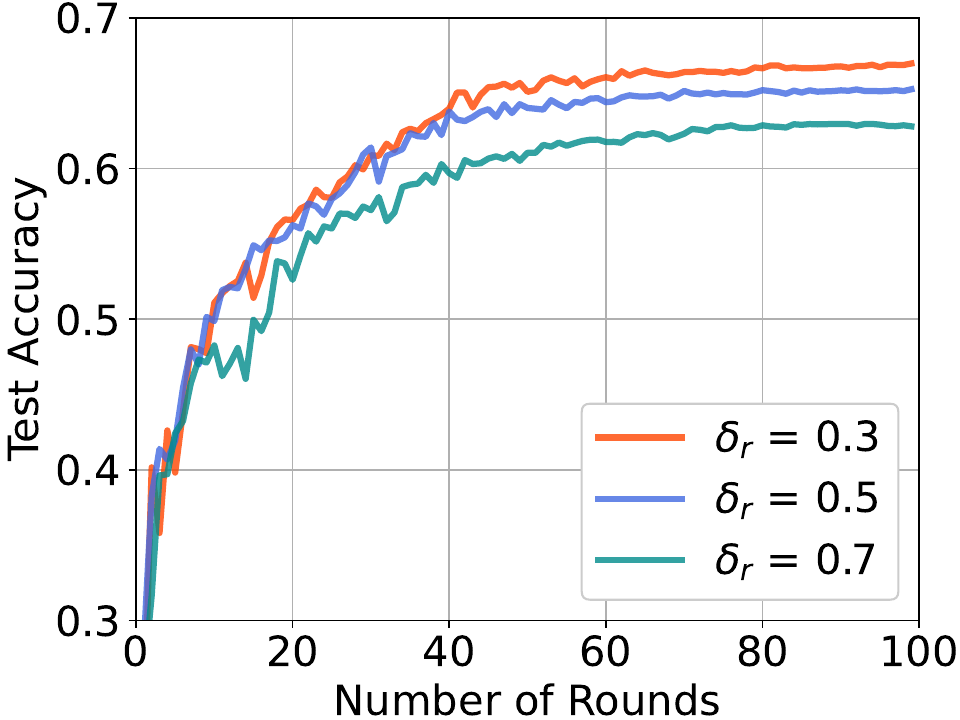}} 
\subfloat[Test Accuracy with MVSA-Single]{\includegraphics[width=0.49\linewidth]{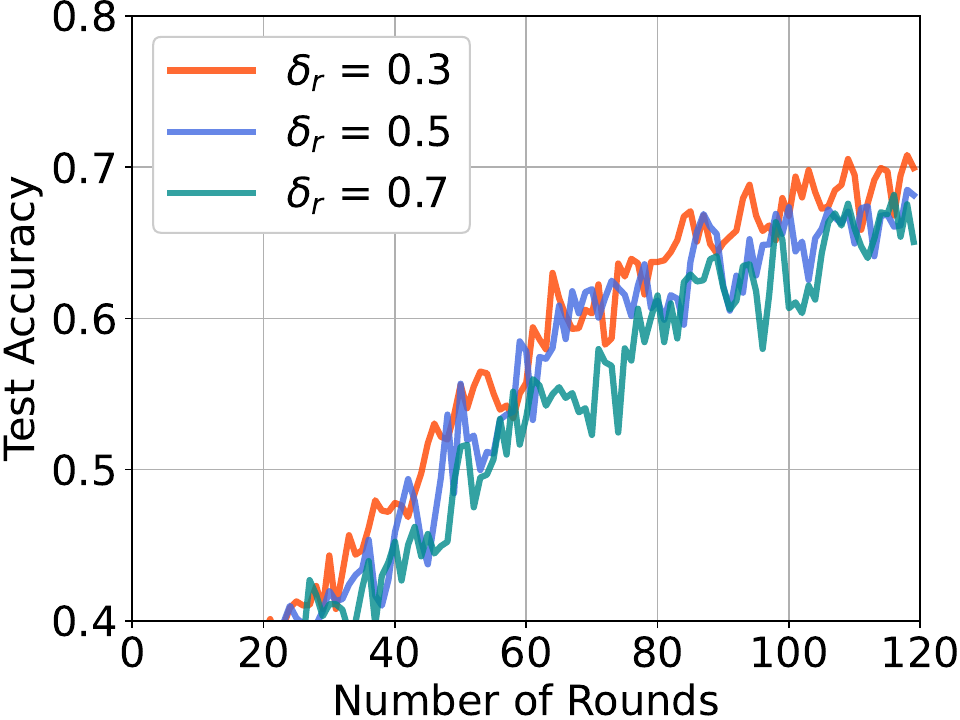}} 
\caption{Performance comparison of proposed algorithm and benchmarks with different $\delta_r$ ratio.}   \label{mmofl-inter-l}
\vspace{-10pt}
\end{figure}

\textbf{Impact of Ratio $\delta_r$ (Quality Imbalance)}. In this part, we will explore the effects with the inter-round quality imbalance ratio $\delta_r$ on learning performance. The results for the UCI-HAR and MVSA-Single datasets are presented in Fig.~\ref{mmofl-inter-l}(a) and Fig.~\ref{mmofl-inter-l}(b), respectively. We observe that learning performance decreases as $\delta_r$ increases. When $\delta_r$ is smaller, fewer data samples are affected by quality imbalance. This not only increases the number of updates to the cumulative global prototypes but also reduces the number of rounds in which the PCE is used to mitigate quality imbalance. Together, these factors contribute to improved learning performance.

Next, we assess the efficiency of the proposed PNR and PLR algorithms by applying quantized upload strategies to further reduce communication overhead and enhance their practicality. We evaluate performance under quantity imbalance and quality imbalance scenarios separately.

\begin{figure}[h]
\centering
\vspace{-15pt}
\subfloat[Test Accuracy with UCI-HAR ]{\includegraphics[width=0.49\linewidth]{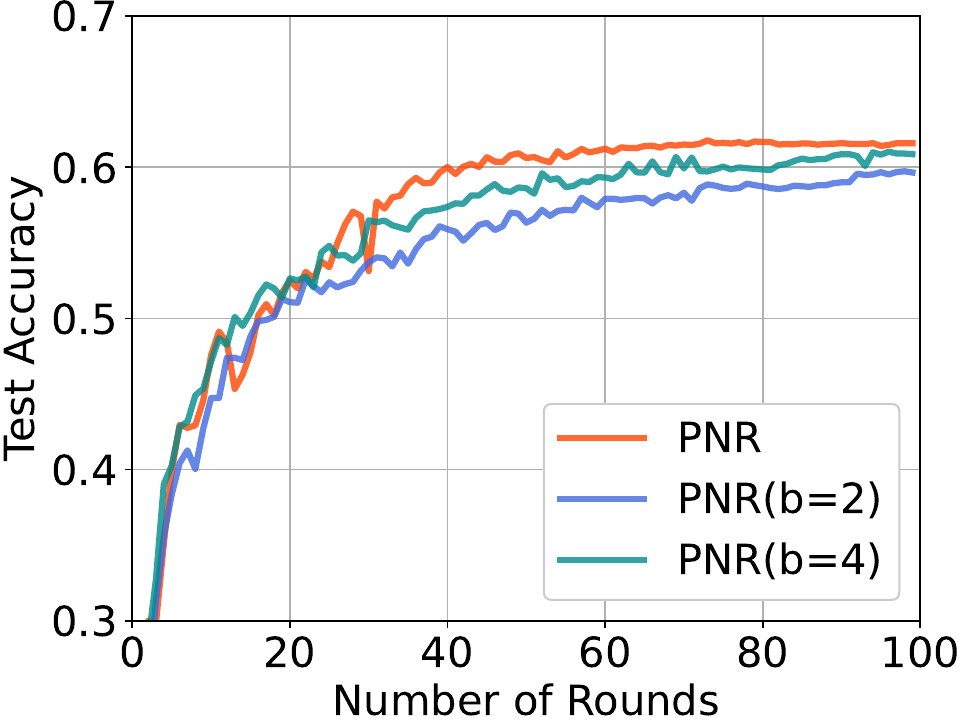}} 
\subfloat[Test Accuracy with MVSA-Single ]{\includegraphics[width=0.49\linewidth]{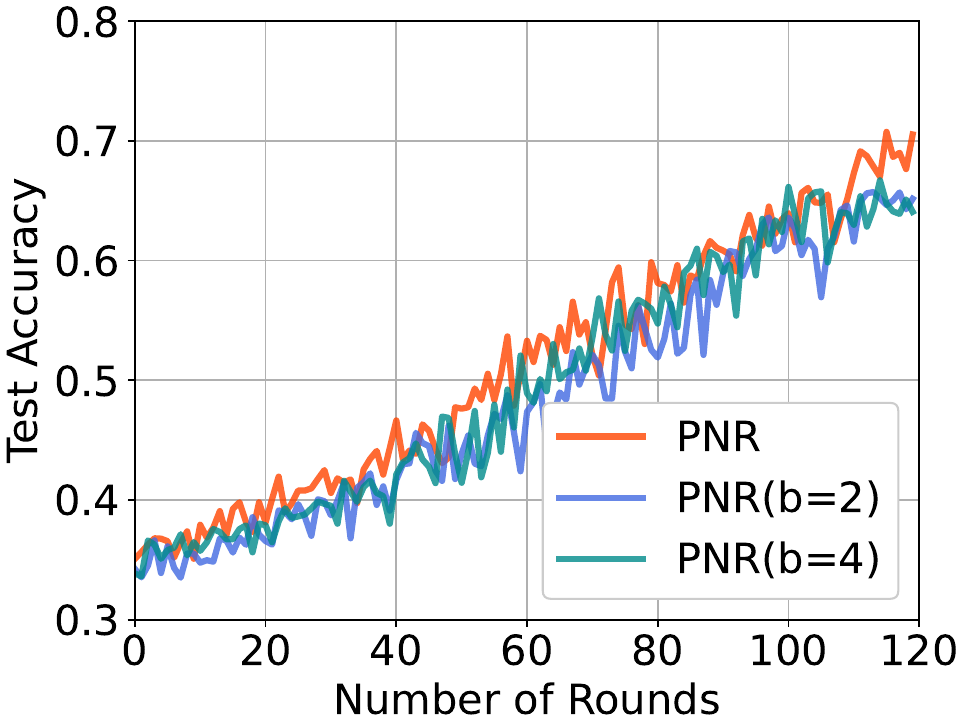}} \\
\subfloat[Test Accuracy with UCI-HAR by communication cost ]{\includegraphics[width=0.49\linewidth]{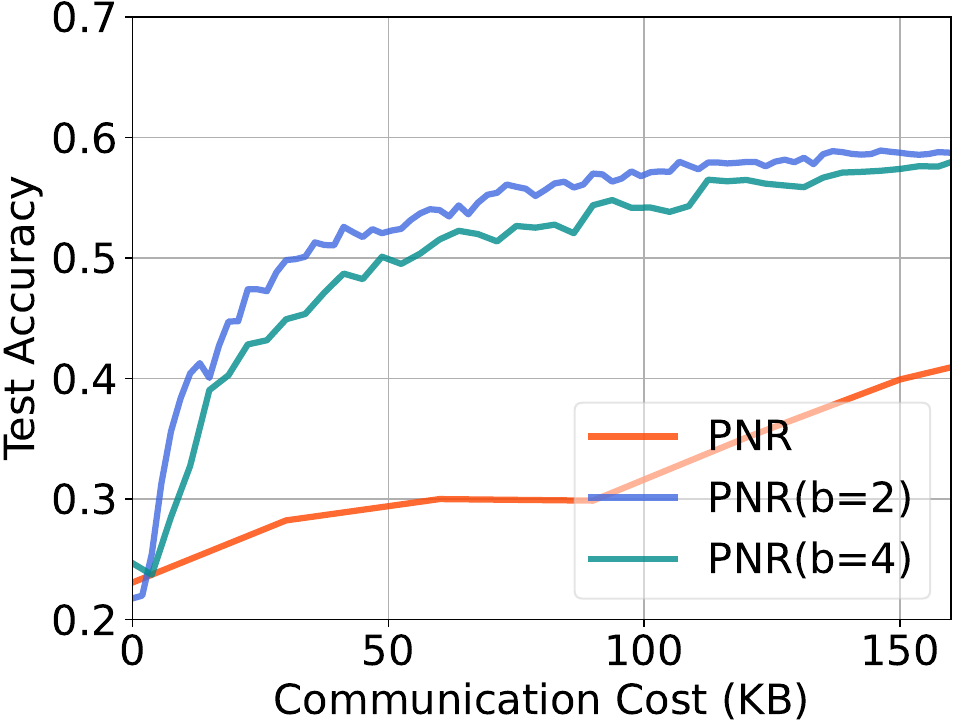}} 
\subfloat[Test Accuracy with MVSA-Single by communication cost]{\includegraphics[width=0.49\linewidth]{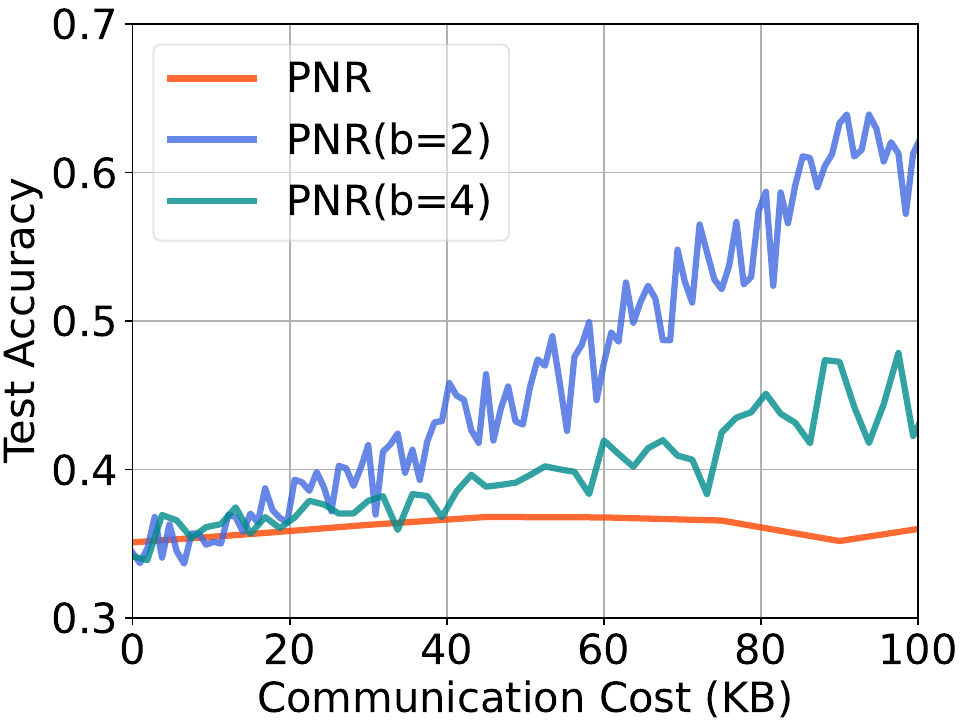}}
\caption{Performance evaluation of the proposed algorithm under varying quantization levels and quantity imbalance.}   \label{mmofl-quanti-qnr}
\vspace{-10pt}
\end{figure}

\textbf{Impact of Quantized Upload (Quantity Imbalance)}. In this part, we examine the effect of applying quantization to local prototypes during the PNR upload process as a means of further reducing communication overhead. We employ a uniform scalar quantizer, where the parameter $b$ represents the number of bits per component used for compression. Without quantization, $b = 32$, corresponding to full-precision transmission. Applying quantization with $b$ bits results in $2^b$ quantization levels. We evaluate the performance impact for $b \in [2, 4]$. The results, as illustrated in Fig.~\ref{mmofl-quanti-qnr}(a) and Fig.~\ref{mmofl-quanti-qnr}(b) by round, and Fig.~\ref{mmofl-quanti-qnr}(c) and Fig.~\ref{mmofl-quanti-qnr}(d) by communication cost, illustrate the trade-off between model performance and communication efficiency. As expected, larger $b$ values yield better performance but incur higher communication costs. Notably, prototypes remain effective at low precision, and quantization leads to only a minor performance drop while substantially reducing communication overhead. This confirms that PNR can be efficiently integrated into the MMO-FL framework without imposing a significant communication burden. Similarly, the PLR results, presented in Fig.~\ref{mmofl-quanti-qlr}(a) and Fig.~\ref{mmofl-quanti-qlr}(b) by round, and Fig.~\ref{mmofl-quanti-qlr}(c) and Fig.~\ref{mmofl-quanti-qlr}(d) by communication cost, exhibit the same trade-off pattern. As with PNR, low-precision quantization in PLR achieves substantial communication savings with minimal performance degradation, leading to the same simulation conclusion.

\begin{figure}[h]
\centering
\vspace{-15pt}
\subfloat[Test Accuracy with UCI-HAR ]{\includegraphics[width=0.49\linewidth]{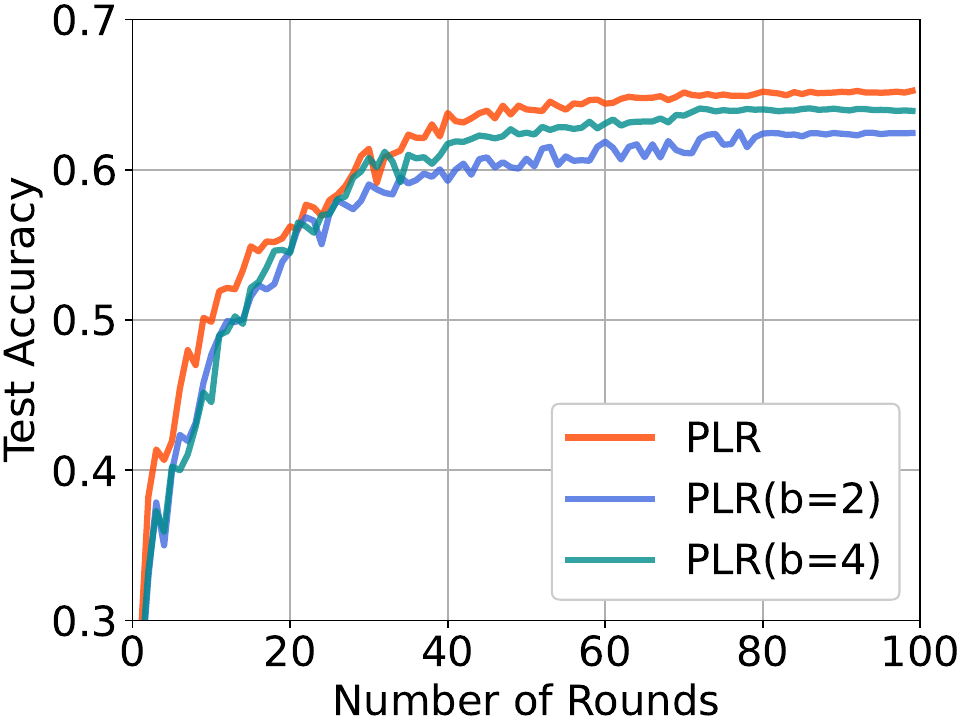}} 
\subfloat[Test Accuracy with MVSA-Single ]{\includegraphics[width=0.49\linewidth]{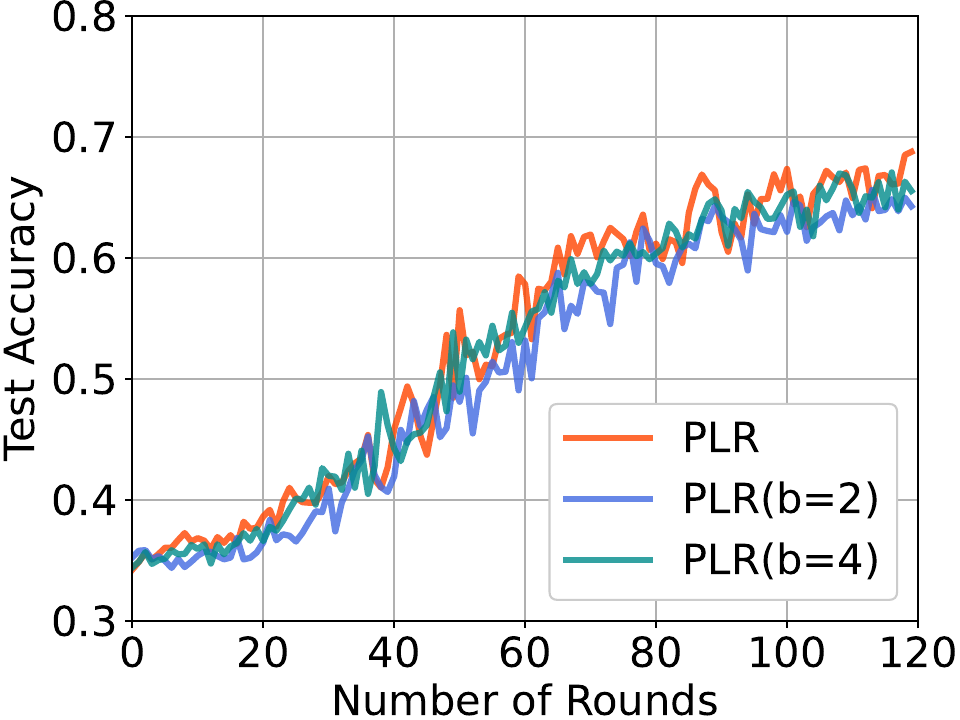}} \\
\subfloat[Test Accuracy with UCI-HAR by communication cost ]{\includegraphics[width=0.49\linewidth]{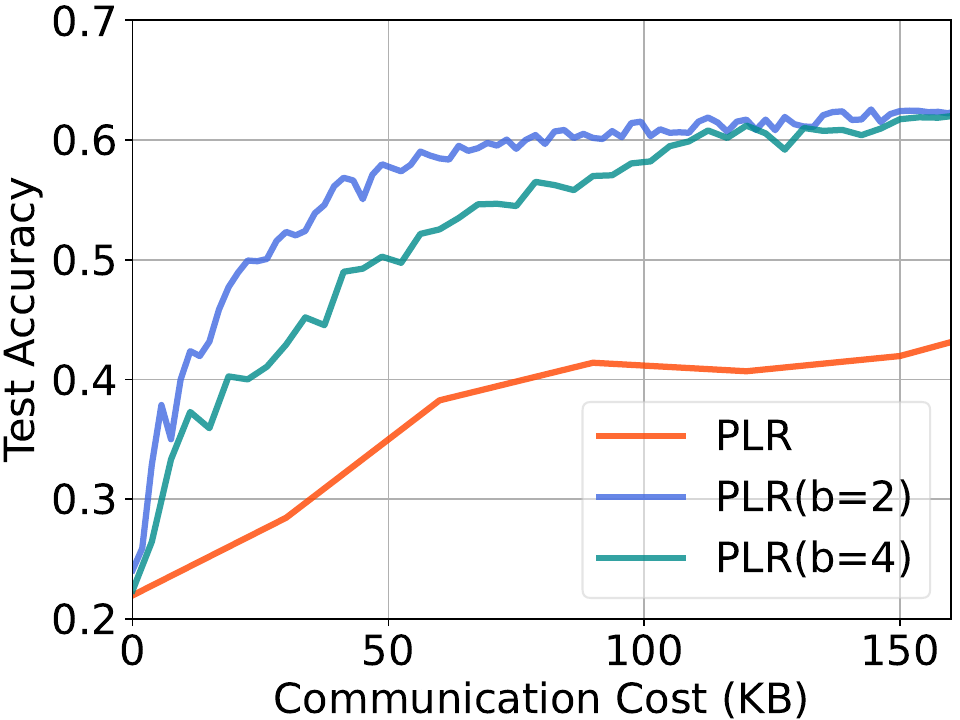}} 
\subfloat[Test Accuracy with MVSA-Single by communication cost ]{\includegraphics[width=0.49\linewidth]{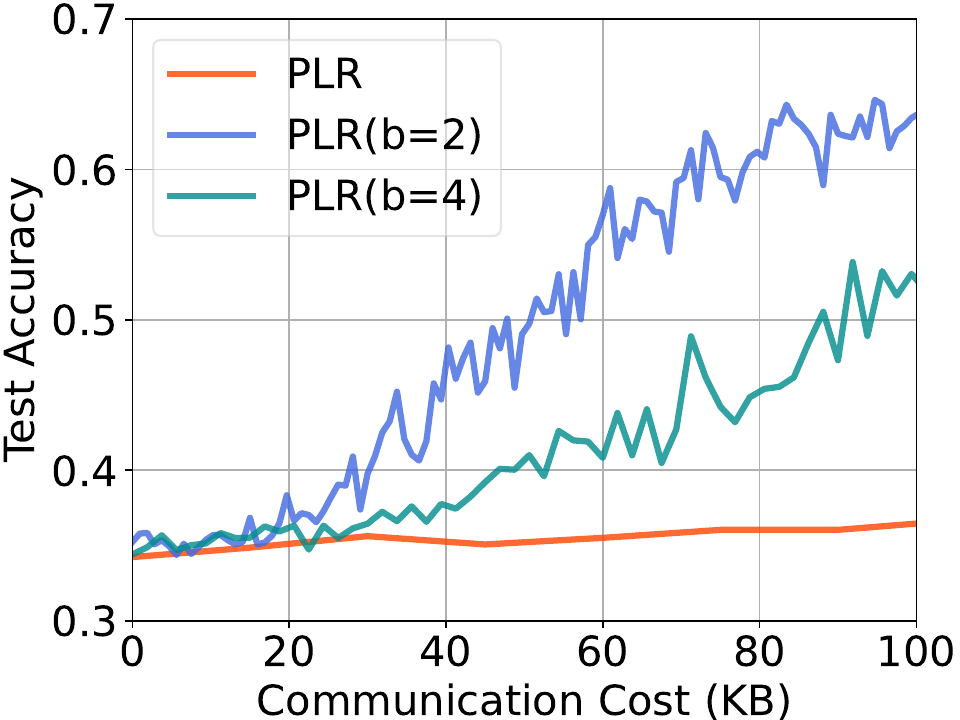}}
\caption{Performance evaluation of the proposed algorithm under varying quantization levels and  quality imbalance.}   \label{mmofl-quanti-qlr}
\vspace{-15pt}
\end{figure}

\section{Conclusion }
In this work, we investigate the MMO-FL framework, which supports distributed multimodal learning with online data collection on edge devices. We focus on non-ideal conditions where modality quantity and quality imbalance arise due to variations in device performance during data acquisition. To address these challenges, we propose the QQR algorithms, designed to mitigate quantitative and qualitative modality imbalances, respectively. The proposed approach is supported by comprehensive theoretical analysis and validated through extensive experimental evaluations. For future work, we aim to advance the theoretical foundations of MMO-FL by exploring additional challenges to further enrich research in this domain. From a practical perspective, we plan to develop a real-world testbed leveraging IoT edge devices for actual data collection and distributed and online model training. This deployment will help uncover additional issues related to modality quantity and quality imbalance that may not be evident in simulation environments, thereby enabling further refinement and optimization of the proposed QQR algorithms.

\bibliographystyle{IEEEtran}
\bibliography{bibligraphy}

\end{document}